\pdfoutput=1

\documentclass[11pt]{article}

\usepackage{acl}

\usepackage{times}
\usepackage{latexsym}

\usepackage[T1]{fontenc}

\usepackage{inconsolata}

\usepackage{blindtext}
\usepackage{hyperref}
\usepackage{url}
\usepackage{pifont}
\newcommand{\cmark}{\ding{51}}%
\newcommand{\xmark}{\ding{55}}%
\usepackage[short]{optidef}
\usepackage{afterpage}
\usepackage[utf8]{inputenc} 
\usepackage[T1]{fontenc}    
\usepackage{hyperref}       
\usepackage{url}            
\usepackage{booktabs}       
\usepackage{amsfonts}       
\usepackage{nicefrac}       
\usepackage{microtype}      
\usepackage{lipsum}
\usepackage{graphicx}
\usepackage{amsthm}
\usepackage{float}
\usepackage{graphics}
\usepackage{graphicx}
\usepackage[skip=3pt]{caption}
\usepackage{algorithm} 
\usepackage{algorithmicx}
\usepackage{algpseudocode}
\usepackage{bm}
\usepackage{color}
\usepackage{multirow}
\usepackage{makecell}
\usepackage{balance}

\newtheorem{proposition}{Proposition}
\usepackage{subcaption}
\usepackage[shortlabels]{enumitem}
\usepackage{amsmath,bm}
\usepackage{wrapfig}

\usepackage{amssymb}
\usepackage{xcolor}
\usepackage{footmisc}
\usepackage{mathrsfs}
\allowdisplaybreaks[4]

\theoremstyle{definition}

\usepackage[utf8]{inputenc} 
\usepackage[T1]{fontenc}    
\usepackage{hyperref}       
\usepackage{url}            
\usepackage{booktabs}       
\usepackage{siunitx}
\usepackage{amsfonts}       
\usepackage{nicefrac}       
\usepackage{microtype}      
\usepackage{xcolor}         
\usepackage{soul}
\usepackage{xspace}
\usepackage{bbding}
\usepackage{lineno}

\newcommand{\ie}{{\emph{i.e.}},\xspace}
\newcommand{\eg}{{\emph{e.g.}},\xspace}

\newcommand{\J}{\mathbf{J}}

\newcommand{\Lc}{\mathcal{L}}
\newcommand{\Oc}{\mathcal{O}}
\newcommand{\M}{\mathbf{M}}

\newcommand{\Mcc}{\mathcal{M}}

\newcommand{\Pc}{\mathcal{P}}

\newcommand{\Rb}{\mathbb{R}}

\newcommand{\Tc}{\mathcal{T}}

\renewcommand{\u}{\mathbf{u}}

\renewcommand{\v}{\mathbf{v}}

\newcommand{\Wc}{\mathcal{W}}

\newcommand{\x}{\mathbf{x}}

\newcommand{\y}{\mathbf{y}}

\newcommand{\Xc}{ \mathcal{X}}
\newcommand{\Yc}{ \mathcal{Y}}

\newcommand{\Qc}{\mathcal{Q}}

\theoremstyle{definition}

\theoremstyle{remark}

\makeatletter
\def\@fnsymbol#1{\ensuremath{\ifcase#1\or \ddagger \or * \or
\mathsection\or \mathparagraph\or \|\or **\or \dagger\dagger
\or \ddagger\ddagger \else\@ctrerr\fi}}
\makeatother
\title{\textbf{Q-Tuning}: Queue-based Prompt Tuning for Lifelong Few-shot Language Learning}

\author{
\textbf{Yanhui Guo}$^1\thanks{\ Equal contribution.}\thanks{\ Work done during internship at Amazon.}$\ \ 
\textbf{Shaoyuan Xu}$^2$\footnotemark[1] \ \ 
\textbf{Jinmiao Fu}$^2$\footnotemark[1] \\ 
\textbf{Jia Liu}$^{2,3}$ \ \ 
\textbf{Chaosheng Dong}$^{2}$\ \ 
\textbf{Bryan Wang}$^{2}$
\\ 
  $^1$McMaster University, Canada 
  \quad
  $^2$Amazon, USA
  \quad
  $^3$The Ohio State University, USA
  \\
  \ \ \ \ {\tt \{guoy143\}@mcmaster.ca}\ \ 
  {\tt \{shaoyux,jinmiaof,hliujia,chaosd,brywan\}@amazon.com} 
}

\begin{document}

\maketitle


\begin{abstract}
This paper introduces \textbf{Q-tuning}, a novel approach for continual prompt tuning that enables the lifelong learning of a pre-trained language model. When learning a new task, Q-tuning trains a task-specific prompt by adding it to a prompt queue consisting of the prompts from older tasks. To better transfer the knowledge of old tasks, we design an adaptive knowledge aggregation technique that reweighs previous prompts in the queue with a learnable low-rank matrix. Once the prompt queue reaches its maximum capacity, we leverage a PCA-based eviction rule to reduce the queue's size, allowing the newly trained prompt to be added while preserving the primary knowledge of old tasks. In order to mitigate the accumulation of information loss caused by the eviction, we additionally propose a globally shared prefix prompt and a memory retention regularization based on information theory. Extensive experiments demonstrate that our approach outperforms the state-of-the-art methods substantially on continual prompt tuning benchmarks. Moreover, our approach enables lifelong learning on linearly growing task sequences while requiring constant complexity for training and inference. 
\end{abstract}



\section{Introduction}\label{intro}

In recent years, pretrained language models (LMs) have achieved huge success in natural language processing \citep{brown2020language,jmfu-CMA-icip,thoppilan2022lamda,jia-etal-2023-kg,openai2023gpt4}, which popularizes the pretraining-finetuning pipeline in applications. 
However, with the ever-growing parameter scale of modern LMs (\eg GPT-4 that may have 1.76 trillion parameters \citep{openai2023gpt4}), it becomes increasingly difficult to finetune the whole model, leading to extensive attention to parameter-efficient finetuning (PEFT) technologies. 
{\em Prompt tuning} (PT) \citep{liu2022p} has recently emerged as a leading PEFT solution.
PT trains soft prompts and prepends them to the input of LMs, while keeping the LM parameters frozen.
Existing works \citep{lester2021power,liu2023pre} have
shown that PT can achieve performance on par with finetuning, while requiring fewer than 0.01\% of the total trainable parameters. Continual prompt tuning (CPT) is a methodology that extends PT to the continual learning (CL) paradigm for learning new tasks that arrive in a {\em sequential} fashion.

CPT encounters technical challenges akin to those faced by traditional CL methods, including the well-known {\em catastrophic forgetting} (CF) \citep{lin2022trgp} and {\em forward knowledge transfer} (FKT). 
To overcome these challenges, \citet{wang2022dualprompt} designed a dual prompt tuning framework including a globally shared prompt and a task-specific prompt. However, continuously optimizing the shared prompt for new tasks will make the learned knowledge from old tasks vanish, leading to less efficient FKT. To improve the FKT, ProgPrompt was proposed by \citet{razdaibiedina2023progressive} which maintains a prompt list for incoming tasks by progressively adding newly trained prompts while storing all previously trained prompts. Following a similar strategy to extending the prompt list, \citet{smith2023coda} proposed a prompt set expansion method by weighting the sum over a group of prompt components for each task. Although these methods succeed in improving the FKT, they suffer from the same problem when the length of prompts grows linearly at a rate of $\Oc(N)$ along with the number of tasks. This leads to an $\Oc(N^2)$ complexity for transformer-based models. Consequently, the training and inference costs will become intractable as $N$ increases and exceeds a finite computation resource limit.


In this paper, we overcome the aforementioned challenges by proposing a novel continual prompt tuning technology named {\em Queue-based prompt tuning} ({\bf Q-tuning}). 
Q-tuning manages a {\em Queue-based prompt} ({\bf Q-prompt}), which is stored in a {\em finite-size} data buffer. 
For learning a new task, Q-tuning trains a new prompt combined with the fixed Q-prompt that stores all previously learned prompts. Upon the completion of tuning, the latest trained prompt will be added to the Q-prompt for the tuning of the next task.  
Once the number of tasks exceeds the queue-size limit, we will remove less informative prompts according to a principal component analysis (PCA) based dequeue rule.
This endows Q-tuning with the ability to perform lifelong prompt tuning on extremely long task sequences.
Our key contributions and results can be summarized as follows:


\begin{list}{\labelitemi}{\leftmargin=1em \itemindent=-0.0em \itemsep=.1em}

\item We propose a continual prompt tuning method called Q-tuning that, to our knowledge, is the first technique to achieve lifelong learning in application scenarios with an agnostic number of new tasks through prompt tuning.




\item Q-tuning consists of a prompt queue (Q-prompt) and an adaptive knowledge aggregation low-rank matrix that is optimized to capture the importance of the enqueued prompts to enhance FKT. A novel dequeue rule based on PCA is applied to trim the Q-prompt when it is full. In addition, a globally shared prefix prompt with a memory retention (MR) technique is devised to mitigate the information loss due to dequeuing.

\item We conduct extensive experiments to demonstrate the successful applications of our proposed Q-tuning on few-shot CL tasks. Q-tuning outperforms all the competing CL methods by a large margin. In addition, Q-tuning highlights its ability to facilitate lifelong learning. For instance, our experiments on extremely long learning sequences consisting of 70 disjoint tasks have shown a 30\% accuracy improvement over the standard prompt tuning method.





\end{list}

\section{Related work \label{Section:related}}  


{\bf 1) Continual Learning:}
Continual Learning (CL), also known as lifelong learning, is to learn from a stream of different tasks arriving sequentially. The major challenge of CL is to prevent the CF problem \citep{kemker2018measuring} and achieve knowledge transfer \citep{ke2021achieving}. 
Existing CL approaches can be divided into three categories: 1) Memory-based methods \citep{Shin2017CLDGR,BangKY0C21,jiao2022fine,ermis2022memory} that store previous data and replay them when training on the next task to mitigate CF issue;
2) Regularization-based methods \citep{kirkpatrick2017overcoming,zenke2017continual,schwarz2018progress} that apply an additional regularization loss to constrain the update of parameters that are less important to learning new tasks;
3) Architecture-based methods that dynamically expand the network capacity \citep{rusu2016progressive,yoon2018lifelong} or train new task-specific parameters \citep{yoon2020scalable} while fixing parameters for old tasks to prevent forgetting. However, these methods require finetuning all model parameters and are too expensive to put into practice for large-scale models with an astronomical number of parameters, such as large language models (LLMs).

{\bf 2) Prompt Tuning:}
Prompt tuning \citep{lester2021power,li2021prefix,gu-etal-2022-ppt,jia2022visual,wang2023multitask} is a lightweight approach to finetune an LLM model for a target task, which only requires optimizing a series of virtual tokens (a.k.a ``soft prompt'') instead of updating the entire model. It has been demonstrated that prompt tuning can achieve the same or even better performance than training a full model. 
In prompt tuning, a trainable soft prompt $\theta_{\Pc}$ is prepended to the input text $\x$ while keeping other parameters frozen. In this case, the combined model parameters include trainable prompt parameters $\theta_{\Pc}$ and parameters $\theta_{\Mcc}$ of a pretrained model $\Mcc$.  Given the task $\Tc = (\Xc,\Yc)$ consisting of training pairs $(\x,\y)$, the objective of prompt tuning is: 
\begin{align}
    \mathop{\max}_{\theta_{\Pc}} \sum_{(\x,\y)\in \Tc}{\log}\, p (\y|\x;\theta_{\Mcc},\theta_{\Pc}).
    \label{eq_loss_prompt}
\end{align}
{\bf 3) Continual Prompt Tuning:}
Many works \cite{zhu2022continual, yin2022contintin, ermis2022memory,wang2022dualprompt,razdaibiedina2023progressive} have applied prompt tuning to the continual learning domain, but we observe some limitations of these methods. 
For example, the techniques proposed by \citet{zhu2022continual, ermis2022memory} require a large data buffer to store training samples from previous tasks for anti-forgetting. The paradigms of progressively extending the prompts \cite{razdaibiedina2023progressive, Wang2022SPromptsLW, smith2023coda} are inapplicable to the scenario with an infinitely increasing number of tasks.

To address the aforementioned limitations, we introduce Q-tuning, which is data-buffer-free and enables anti-forgetting lifelong learning in the face of an ever-expanding stream of tasks.

\algnewcommand\algorithmicinput{\textbf{Input:}}
\algnewcommand\algorithmicoutput{\textbf{Output:}}
\algnewcommand\algorithmicinitialization{\textbf{Initialize:}}
\algnewcommand\Input{\item[\algorithmicinput]}%
\algnewcommand\Output{\item[\algorithmicoutput]}%
\algnewcommand\Initialize{\item[\algorithmicinitialization]}%

\section{The Q-Tuning Approach}\label{Section: algorithm}

\begin{figure*}[htbp]
\centering
\includegraphics[width=1\textwidth]{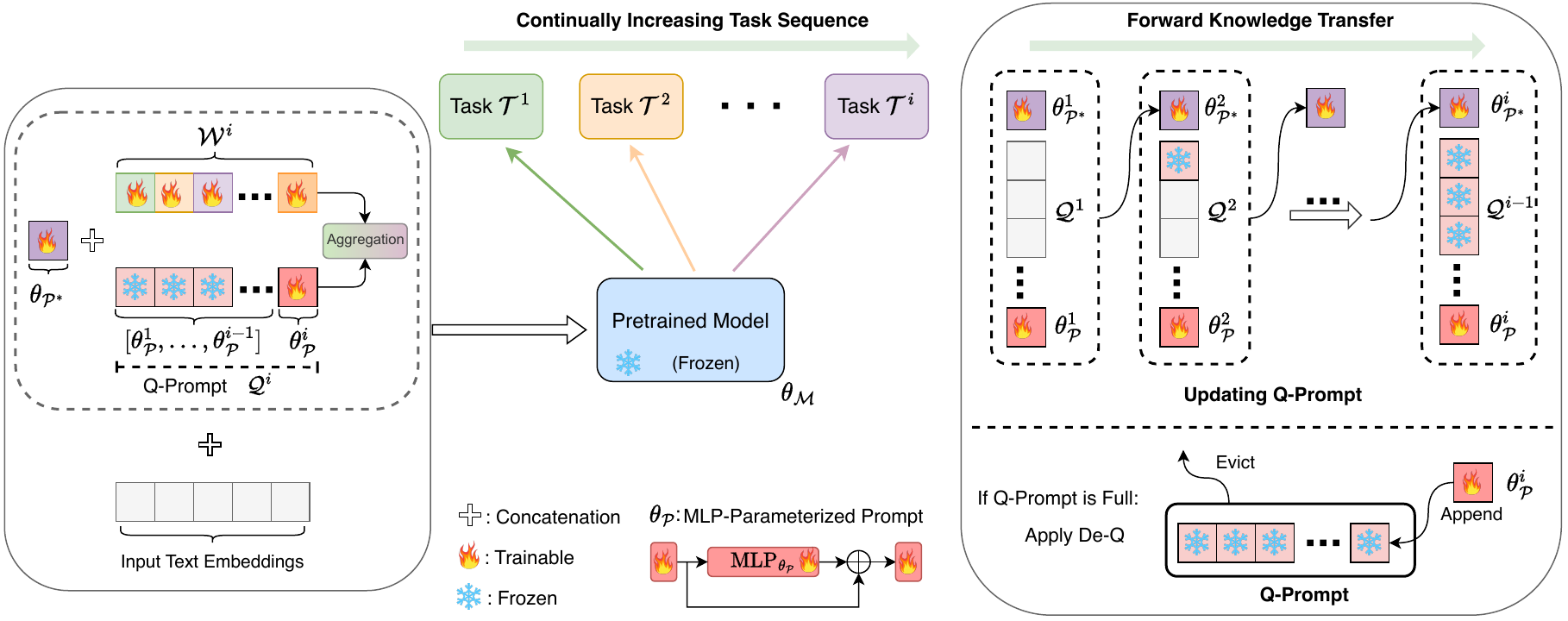}
\captionof{figure}{\footnotesize The overall framework of the proposed \textbf{Q-tuning} technology. Given a continually growing-up task sequence, we propose a prompt queue (Q-prompt) and a globally {\em shared} prefix prompt $\theta^i_{\Pc^{\ast}}$ to achieve the forward knowledge transfer, where the superscript of $\theta^i_{\Pc^{\ast}}$ denotes the $i$-th status. Moreover, we adopt a knowledge aggregation method to adaptively adjust the contribution of each fixed prompt $[\theta^1_{\Pc},\theta^2_{\Pc},\ldots,\theta^{i-1}_{\Pc}]$ in Q-prompt by using a rank-one matrix $\Wc^i$. We parameterize the trainable soft prompt by a two-layer residual MLP. If the length of the Q-prompt exceeds the limit, we apply a De-Q rule to discard less informative prompts in the queue.}
\label{fig:overall_framework}
\end{figure*}




\subsection{Q-prompt and Update Rule \label{sec_Continual_q_tuing}} 
\paragraph{Q-prompt:} 

Given the continually increased task set $\Tc = \{(\Xc^1,\Yc^1),(\Xc^2,\Yc^2),\ldots,(\Xc^i,\Yc^i)\}$, where $\mathcal{T}^i=(\Xc^i,\Yc^i)$ denotes the training pairs on $i$-th task, a naive solution is maintaining an increased prompt list \cite{razdaibiedina2023progressive} $[\theta^1_{\Pc},\theta^2_{\Pc},\ldots,\theta^i_{\Pc}]$, where $[\cdot, \cdot]$ is the concatenation operation. 
The objective for the $i$-th task is:

\begin{align}
{\fontsize{9}{0}\selectfont
   \mathop{\max}_{\theta^i_{\Pc}} \sum_{(\x^i,\y^i)\in \Tc^i}{\log}\, p (\y^i|\x^i;\theta_{\Mcc}, \underbrace{[\theta^1_{\Pc},\theta^2_{\Pc},\ldots,\theta^i_{\Pc}]}_{\text{prompt list for $N$ tasks}}).
}
    \label{eq_loss_cl_prompt}
\end{align}
For each task, only the newly appended prompt is trainable, while the previously trained prompts are fixed. However, 
when $N$ grows asymptotically (\ie the model is set as a lifelong learner), training the extremely long prompt list becomes intractable due to the finite system resources. This motivates us to propose the Q-tuning technique. 

Fig.~\ref{fig:overall_framework} illustrates the overall framework of the Q-tuning technique.
In Q-tuning, we add a new trainable prompt to a prompt queue $\Qc$ that stores all previously trained prompts for old tasks. This updated $\Qc$ associated with a globally shared prefix prompt will be tuned for the new task while keeping the prior prompts in $\Qc$ frozen. This progressively appending approach enables forward knowledge transfer as the old task's information is saved in the Q-prompt. 
We let $C=l\times\Qc_{\text{size}}$ denote the maximum capacity of the Q-prompt $\Qc$, where $l$ is the length of a single prompt per task and $\Qc_{\text{size}}$ is the maximum number of prompts in the queue. 
When reaching the capacity limit of $\Qc$, the Q-prompt will be trimmed using an eviction rule to remove less informative prompts and append new trainable prompts for future tasks.




\vspace{-0.1cm}
\paragraph{Knowledge Aggregation for FKT:} In Q-tuning, all prompts in the memory (\ie the Q-prompt $\Qc$), as well as the pretrained LM model, are frozen when learning a new task. 
Consequently, the LM model will be forced to take these fixed prompts in the queue as inputs without considering their relevance to the current task, leading to suboptimal performance. 
To address this problem, we propose a prompt-based knowledge aggregation mechanism. 
For task $i$, we use a trainable matrix $\Wc^i \in \Rb^{c^i\times d}$, which is of the same dimension as the  Q-prompt $\Qc^i$, to scale $\Qc^i$ by $\Wc^i\circ \Qc^i$ ($\circ$ denotes the Hadamard product). Here, for task $i$, we denote the total prompt length of $\Qc^i$ by $c^i = l\times i$. 
Since directly optimizing a large-scale matrix of size $c^i \times d$ is costly, we propose a low-rank multiplicative method inspired by \citet{aghajanyan2021intrinsic}. 
The weight matrix $\Wc^i$ can be expressed as $\Wc^i = \u_i \otimes \v_i^{\mathrm{T}}$, where $\u_i \in \Rb^{c^i}$, $\v_i \in \Rb^{d}$ and $\otimes$ denotes the outer product. 
Clearly, $\Wc^i$ is a rank-one matrix, and the number of trainable parameters is reduced to $c^i+d \ll c^i\times d$. We jointly optimize the newly appended prompt $\theta^i_{\Pc}$ and the low-rank aggregation matrix $\Wc^i$ by maximizing the cross-entropy loss as follows: 
\begin{align}
    \mathop{\max}_{\theta^i_{\Pc},\Wc^i}  \sum_{(\x^i,\y^i)\in \Tc^i}{\log}\, p (\y^i|\x^i;\theta_{\Mcc},&\nonumber\\
    \Wc^i \, \circ \underbrace{\Qc^{i}(\theta^1_{\Pc},\cdots, \theta^i_{\Pc})}_{\text{maximum length is }C = l\times \Qc_{\text{size}}}),&
\label{w_qtuning_loss}
\end{align}
where only the new added prompt $\theta^i_{\Pc}$ and the weight matrix $\Wc^i$ for the $i$-th task are trainable.
\paragraph{De-Q Rule:} 
Our Q-prompt design allows appending newly trained prompts until they reach the maximum length. 
Once the Q-prompt is full (denoted by $\Qc_C$), a dequeuing (De-Q) rule is executed to reduce the length of $\Qc_C$ to $C-l$ to add the new prompt.
However, this leads to a challenge: {\em how to retain the most useful prompt information after trimming the Q-prompt?} Straightforward De-Q rules include random eviction and first in first out (FIFO). However, these simple rules may discard valuable information in the queue, resulting in negative impacts on FKT. 


To address the above problem, we introduce a simple yet effective De-Q rule named DQ-PCA based on principal component analysis (PCA) \citep{shlens2014tutorial}. 
Specifically, we first calculate the centered Q-prompt $\tilde{\Qc}_C \in \Rb^{C\times d}$ with a zero mean: $\tilde{\Qc}_C = \Qc_C - \text{mean}(\Qc_C)$. Then we perform singular value decomposition (SVD). We extract the first $C-l$ principal components to obtain the trimmed Q-prompt $\tilde{\Qc}_{C-l}\in \Rb^{(C-l)\times d}$ and enqueue the new trainable $\theta^i_{\Pc} \in \Rb^{l\times d}$. 
This process can be written as follows:
\begin{align}
 \text{SVD}(\tilde{\Qc}_C) &= U \Sigma V^{\mathrm{T}},\ 
 \tilde{\Qc}_{C-l} = \Sigma_{C-l}V^{\mathrm{T}}_{C-l},\
 \\
 &\Qc_C \xleftarrow{\text{Update}}  \tilde{\Qc}_{C-l} \oplus \theta^i_{\Pc},
\label{eq_enqueue}
\end{align}
where $\oplus$ denotes the concatenation operation $[\tilde{\Qc}_{C-l},\theta^i_{\Pc}]$,  $U \in \Rb^{C\times C}$ is the matrix consisting of the left singular vectors, $\Sigma \in \Rb^{C\times d}$ is the diagonal matrix formed by the singular values in decreasing order and $V^{\mathrm{T}}$ 
is the matrix of right singular vectors. The matrix 
$V_{C-l}^{\mathrm{T}}$ 
is formed by the top $C-l$ principle row vectors of $V^{\mathrm{T}}$ and $\Sigma_{C-l}\in \Rb^{(C-l)\times (C-l)}$ 
denotes the diagonal matrix with the top $C-l$ singular values. 
When the length of the Q-prompt exceeds $C$, it will trigger the DQ-PCA to shrink the Q-prompt's length to $C-l$. As a result, Q-tuning achieves an $\Oc(1)$ training and inference complexity instead of $\Oc(N^2)$ for transformer-based LMs, thereby enabling low-cost lifelong learning\footnote{For example, on a single NVIDIA V100 GPU (32GB) with the same training setting as ProgPrompt \citep{razdaibiedina2023progressive}, Q-tuning can easily handle an extremely long 70-task sequence, while ProgPrompt fails due to memory overflow (cf. our experiments).}.


\subsection{Prefix Prompt for Knowledge Sharing \label{sec_shared_prompt}}
Although DQ-PCA is able to minimize information loss by keeping the most useful information from previous prompts in the queue, information loss will inevitably accumulate as the number of tasks grows larger. 
To avoid such loss, we introduce a globally shared prefix prompt, $\theta_{\Pc^{\ast}}$. This prefix prompt is appended to the head of the Q-prompt and continually trained across all the tasks for aggregating the global information. 
However, continuously training the shared prompt $\theta_{\Pc^{\ast}}$ will make it lean toward the newest task. 
To address this limitation, we propose a \textit{memory retention} (MR) regularization by maximizing the overlapping information between the shared prefix prompt and the learned knowledge from old tasks. 
For each task $i$, we formulate the maximization problem as:
\begin{align}
   \mathop{\max}_{\theta^i_{\Pc^{\ast}}}\ & \ I(\underbrace{p(\y^i|\x^i;\theta_{\Mcc},\
   \theta^i_{\Pc^{\ast}})}_{p({\xi^i})};\  \nonumber
   \\
   &\underbrace{p(\y^i|\x^i;\theta_{\Mcc}, \Wc^{i-1} \circ [\theta^{i-1} _{\Pc^{\ast}},\Qc^{i-1}])}_{p({\xi^{i-1}})}),
   \label{loss_MI_memory}
\end{align}
where $I(\cdot,\cdot)$ represents the mutual information between two random variables, $\theta^i_{\Pc^{\ast}}$ denotes the shared prompt to be learnt for $i$-th task, $\theta^{i-1}_{\Pc^{\ast}}$ is the shared prompt learnt until task $i-1$, and $\Qc^{i-1}$ denotes the Q-prompt until task $i-1$. 
The second term $p({\xi^{i-1}})$ in Eq.~(\ref{loss_MI_memory}) represents the old knowledge learnt before the $i$-th task, provided by the shared  $\theta^{i-1}_{\Pc^{\ast}}$ and the Q-prompt $\Qc^{i-1}$. 
Maximizing Eq.~(\ref{loss_MI_memory}) can transfer the old knowledge modeled by $p({\xi^{i-1}})$ to current shared prompt $\theta^i_{\Pc^{\ast}}$. Such regularization can mitigate the information loss caused by trimming $\Qc^{i-1}$ when the Q-prompt $\Qc^{i-1}$ at task $i-1$ reaches its maximum length. As a result, the full information prior to the new task $i+1$ can be represented by the union of $\Qc^{i}$ and $\theta^{i}_{\Pc^{\ast}}$.

To solve the mutual information $I(p(\xi^i);p({\xi^{i-1}}))$ in  Eq.~(\ref{loss_MI_memory}), we adopt the mutual information estimator\footnote{More details about the deviation of the mutual information estimator can be found in Appendix~\ref{appendix:mutual_information}.
} \citep{hjelm2018learning,poole2019variational} based on the Jensen-Shannon divergence (JSD), which satisfies:
\begin{align}
I&(p(\xi^i);p({\xi^{i-1}})) := \mathcal{D}_{\text{JSD}}(\boldsymbol{\J};\boldsymbol{\M})\nonumber
\\
&\geq \mathbb{E}_{z \sim \boldsymbol{\J}}\left[-\sigma(-\mathcal{F}_{\omega}(z))\right] - \mathbb{E}_{z^{\prime}\sim \boldsymbol{\M}}\left[\sigma(\mathcal{F}_\omega(z^{\prime}))\right],
\label{mi_jsd}
\end{align}
where $\boldsymbol{\J} = p(\xi^i,\xi^{i-1})$ and $\boldsymbol{\M}=p(\xi^i)p(\xi^{i-1})$ are the joint and the product of marginals of the random variables $\xi^i$ and $\xi^{i-1}$, respectively, and $\sigma(t)=\mathrm{log}(1+e^t)$. 
$\mathcal{F}_{\omega}$ is a discriminator function \citep{nowozin2016f} modeled by an auxiliary neural network with parameters $\omega$. 




\subsection{Objective Function of Q-Tuning \label{sec_mi}}




Given the $i$-th classification task, the training objective of Q-tuning is defined as:
\begin{align}
\Lc_{\Qc}(\theta^i_{\Pc^{\ast}},\theta^i_{\Pc},\Wc^i) = -\sum_{(\x^i,\y^i)\in \Tc^i}{\log}\, p (\y^i|\x^i;\theta_{\Mcc},&\nonumber
\\
\theta^i_{\Pc^{\ast}},
\Wc^i \circ \Qc^i(\theta^1_{\Pc},\cdots,\theta^i_{\Pc})),&
\end{align}
where $\Tc^{i}$ denotes the data streams of the $i$-th task. The pretrained model $\theta_{\Mcc}$ and all the enqueued prompts prior to $i$-th task are fixed. The trainable parameters include the shared prefix prompt $\theta^i_{\Pc^{\ast}}$, the newly appended prompt $\theta^i_{\Pc}$ and the queue aggregation matrix $\Wc^i$. 

For the prefix prompt $\theta^i_{\Pc^{\ast}}$, we enable its capability for memorizing the knowledge of old tasks with the MR regularization defined by Eq.~(\ref{loss_MI_memory}). According to Eq.~(\ref{mi_jsd}), we can maximize the lower bound of the mutual information, which can be rewritten as minimizing a loss $\Lc_{\text{MR}}$ with respect to $\theta^i_{\Pc^{\ast}}$:
\begin{align}
  \Lc_{\text{MR}}(\theta^i_{\Pc^{\ast}}) &= -\mathbb{E}_{z \sim \boldsymbol{\J}}\left[-\sigma(-\mathcal{F}_{\omega}(z))\right] \nonumber
  \\
  &+ \mathbb{E}_{z^{\prime}\sim \boldsymbol{\M}}\left[\sigma(\mathcal{F}_\omega(z^{\prime}))\right],
  \label{eq_loss_MR}
\end{align}
where $\boldsymbol{\J}$ and $\boldsymbol{\M}$ are defined in Eq.~(\ref{loss_MI_memory}) and Eq.~(\ref{mi_jsd}). The MLP-based discriminator $\mathcal{F}_\omega(\cdot)$ consists of two 512-unit hidden layers. To optimize Eq.~(\ref{eq_loss_MR}) on a given finite training data set, we approximate the expectations using minibatch samples as in \citet{belghazi18a}. 

Putting all things together, we obtain the overall loss:
\begin{align}
   \Lc_{total} = \Lc_{\Qc}(\theta^i_{\Pc^{\ast}},\theta^i_{\Pc},\Wc^i)+ \eta \Lc_{\text{MR}}(\theta^i_{\Pc^{\ast}}),
   \label{eq_loss_total}
\end{align}
where $\eta$ is called ``memory factor'' which is used to weigh the contribution of $\Lc_{\text{MR}}$. When the number of tasks $N\leq C$, we set $\eta=0$, whereas if $N> C$, we set $\eta>0$. We empirically find the best $\eta$ as reported in Table~\ref{tab:eta_ablation} of Appendix~\ref{appendix:more_results}. Algorithm~\ref{q_tuning_algorithm} summarizes the Q-tuning algorithm.

\newcommand{\tabincell}[2]{\begin{tabular}{@{}#1@{}}#2\end{tabular}}  

\section{Experiment Settings}

\subsection{Datasets and Baseline Methods}
\paragraph{Datasets:} 
Following \citet{razdaibiedina2023progressive}, we evaluate the proposed Q-tuning using two \textit{few-shot} CL benchmark settings including short-sequence experiments, long-sequence experiments, and lifelong learning experiments. 

In the short-sequence CL experiments, we adopt 
five text classification datasets \cite{zhang2015character}, including YP reviews, Amazon reviews, DBpedia, Yahoo Answers, and AG News. 
To validate our method's efficacy on different model backbones, we adopt the T5-large model and the BERT-base model for evaluation. For the T5 experiments, 
we use three different orders (\ie Orders 1$\sim$3\footnote{\label{detailsoforders}The details of each order are reported in Table~\ref{tab:seq} of the Appendix. For each order, as in \citet{razdaibiedina2023progressive}, we train three versions of models, with 16 (or 20), 200, and 1000 training samples per class respectively, and report the performance on the test sets correspondingly.}) composed of the AG News, Amazon, Yahoo, and DBpedia datasets 
by following the few-shot CL setting as in \citet{qin2021lfpt5,razdaibiedina2023progressive}. 
For the BERT experiments, we use four different orders (\ie Orders 4$\sim$7\footref{detailsoforders}) including all the above five tasks, 
and we use the same train and test split as IDBR \cite{huang2021continual} including 115,000 training and 7,600 test examples. 

In the long-sequence CL experiments, following \citet{razdaibiedina2023progressive}, we choose 15 different tasks, which consist of the aforementioned five datasets from the short-sequence CL benchmark, four tasks from GLUE benchmark (MNLI, QQP, RTE, SST2) by \citet{wang2018glue}, five tasks from SuperGLUE benchmark by \citet{wang2019superglue} (WiC, CB, COPA, MultiRC, BoolQ), and IMDB movie reviews dataset \citep{maas2011learning}. 
We use three different orders (\ie Orders 8$\sim$10\footref{detailsoforders}). 

Lastly, to mimic the lifelong learning scenario, we further add the Banking77 dataset \cite{casanueva2020efficient}, the Emotion dataset \citep{saravia2018carer}, the rest datasets (WNLI, COLA and QNLI ) of the GLUE benchmark, and WSC of the SuperGLUE benchmark. 
We construct a benchmark with a long sequence of 70 tasks by splitting the datasets with over 4 classes into \textit{disjoint} subsets\footnote{Please refer to Appendix~\ref{appendix:datasets} and Appendix~\ref{appendix:orders}.}. 
Following \citet{razdaibiedina2023progressive}, for each task, we randomly select 500 samples per class from the training set for validation, and use early stopping based on the validation accuracy. 




\paragraph{Baseline Methods for Comparison:} 
In the experiments, we compare our model with 11 baseline methods including: 
(1) \textbf{Per-task Finetune} \citep{razdaibiedina2023progressive}, 
(2) \textbf{Continual Finetune} \citep{wang2020efficient, huang2021continual}, 
(3) \textbf{Prompt Tuning} \citep{qin2021lfpt5,lester2021power}, 
(4) \textbf{Data Replay} \citep{de2019episodic},
(5) \textbf{EWC} \citep{kirkpatrick2017overcoming}, 
(6) \textbf{A-GEM} \citep{chaudhry2018efficient}, 
(7) \textbf{LFPT5} \citep{qin2021lfpt5}, 
(8) \textbf{MBPA++} \citep{de2019episodic}, 
(9) \textbf{IDBR} \citep{huang2021continual},
(10) \textbf{Per-task Prompt} \citep{lester2021power}, and  
(11) \textbf{ProgPrompt} \citep{razdaibiedina2023progressive}\footnote{More introductions to these competing methods are provided in Appendix~\ref{appendix:experiment} due to space limitation.}.



\begin{table}[H]
\fontsize{9}{6}\selectfont
\begin{subtable}[c]{.48\textwidth}
  \centering
  \scalebox{0.88}{
  \begin{tabular}{lccccc} 
    \toprule
      &   & \multicolumn{3}{c}{\textbf{Order}}  &  \\
      \textbf{Method}  & DR & \textbf{1} & \textbf{2}  & \textbf{3}  & \textbf{avg} \\
      \midrule
      Per-task Finetune$^{\ddag}$ &    & 70.0 & 70.0 & 70.0 & 70.0 \\
      Continual Finetune$^{\Box}$  & & 18.9 & 24.9 & 41.7 & 28.5  \\
      Data Replay   &   \checkmark &  35.4 & 37.1 & 41.5 & 38.0  \\
      EWC$^{\Box}$  &   &  39.0 & 38.0 & 44.8 & 40.6  \\
      LFPT5$^{\ast \Box}$ &   \checkmark &  47.6 & 52.6 & 57.9 & 52.7 \\
      ProgPrompt$^{\ast}$ &   &  74.1 & 74.2 & 75.3  & 74.5  \\
      \midrule 
      Ours$^{\ast}$ &   &  \textbf{75.8} & \textbf{75.8} & \textbf{76.9} & \textbf{76.2}\\
     \bottomrule
    \end{tabular}
    }
    \subcaption{Results with the T5-large model.}
    \label{tab:table_t5}
\end{subtable}
\begin{subtable}[c]{0.48\textwidth}
\vspace{5pt}
  \centering
  \scalebox{0.82}{
    \begin{tabular}{lcccccc} 
    \toprule
      & & \multicolumn{4}{c}{\textbf{Order}} & \\
      \textbf{Method} & DR & \textbf{4} & \textbf{5}  & \textbf{6}  & \textbf{7} & \textbf{avg} \\
      \midrule
      Per-task Finetune$^{\ddag}$ &  & 73.9 & 73.9 & 73.9 & 73.9 & 73.9 \\
      Continual Finetune$^{\diamondsuit}$ & & 14.8 & 27.8 & 26.7 & 4.5  & 18.4 \\
      Data Replay$^{\diamondsuit}$ & \checkmark  & 67.2 & 64.7 & 64.7 & 44.6 & 57.8 \\
      A-GEM$^{\diamondsuit}$  & \checkmark & 70.6 & 65.9 & 67.5 & 63.6 & 66.9 \\
      MBPA++$^{\diamondsuit}$ & \checkmark  & 70.8 & 70.9 & 70.2 & 70.7 & 70.6 \\
      IDBR$^{\dag}$  & \checkmark & 75.9 & 76.2 & 76.4 & 76.7 & 76.3 \\
      ProgPrompt$^{\ast}$ & & 77.8 & 77.5 & 77.6 & 77.4 & 77.6\\
      \midrule
      Ours$^{\ast}$ &   &  \textbf{78.5} & \textbf{78.3} & \textbf{78.3}  & \textbf{78.4} & \textbf{78.4} \\
      \bottomrule
    \end{tabular}
    }
    \subcaption{Results with the BERT-base model.}
    \label{tab:table_bert}
\end{subtable}
\caption{Summary of the results with T5 and BERT models on the short-sequence benchmark\protect\footnotemark. Average accuracy after training on the last task is reported. All results are averaged over 3 runs. For the T5 model, we follow few-shot CL settings as in \citet{qin2021lfpt5}.}
\end{table}

\footnotetext{Methods marked with $^{\ast}$ use soft prompt, while other methods train the entire model. For ProgPrompt, the results are reported by running their released code. 
DR denotes whether the data replay is required. $^{\Box}$, $^{\diamondsuit}$, $^{\dag}$ and $^{\ddag}$ mark the results from \citet{qin2021lfpt5}, \citet{de2019episodic}, \citet{huang2021continual}, and \citet{razdaibiedina2023progressive}, respectively.}

\subsection{Implementation Details}
Q-tuning is a model-backbone-agnostic approach that is applicable to any language model, such as the GPT series \citep{openai2023gpt4}, regardless of their sizes. 
Due to resource constraints, following \citet{razdaibiedina2023progressive}, we use two popular language models including the encoder-decoder T5 model \citep{raffel2020exploring} and encoder-only BERT model \citep{devlin2018bert} in our experiments. 
For all the T5 experiments, 
we adopt the T5-large model with the text-to-text formulation, where classification labels are mapped into words (e.g. 0/1 will be mapped as ``True''/``False''). 
For all the BERT experiments, we use the BERT-base model as in IDBR and MBPA++ methods \cite{huang2021continual, de2019episodic}. 
We use the representation of the first token $h_{[CLS]}$ to predict the class of the input text, where $h_{[CLS]}$ is encoded by a beginning-of-a-sentence symbol [CLS].  
Following \citet{razdaibiedina2023progressive}, we apply a linear head including a linear transformation parameterized by $\alpha$ and a softmax function to obtain the classification probabilities over classes 
$k \in \{1 ... \mathcal{K}\}$:
$ p(y = k | h) = \frac{\exp (\alpha_k h_{[CLS]})}{\sum_{y \in \mathcal{K}} \exp (\alpha_y h_{[CLS]})} $. 
The linear head is trained separately for each task.  
In our experiments, the prompt length per task is set to 10, and each prompt is parameterized by a two-layer  MLP\footnote{The experimental details are reported in Appendix~\ref{appendix:experiment}.}.


\section{Experimental Results}
We report Q-tuning performance on T5-large and BERT-base models and compare it to previous CL and prompt tuning approaches. 
We evaluate the methods after training on all tasks and report the averaged test set accuracy across all tasks. The detailed experimental metrics are reported in Appendix \ref{appendix:datasets}. All the experiments are conducted on a single 32GB NVIDIA V100 GPU.

\begin{table*}[htbp]
\centering
\fontsize{9}{0}\selectfont
\small
\scalebox{0.85}{
\begin{tabular}{cc|cccc|cccc}
\toprule
 \multicolumn{2}{c|}{\multirow{2}{*}{\textbf{Method}}}&\multicolumn{4}{c|}{\multirow{1}{*}{\textbf{T5-large}}}&\multicolumn{4}{c}{\textbf{BERT-base}}\\
& & \textbf{Order 8} & \textbf{Order 9}  & \textbf{Order 10} & \textbf{Average} & \textbf{Order 8} & \textbf{Order 9}  & \textbf{Order 10}& \textbf{Average} \\
\midrule
\multicolumn{2}{c|}{Continual Finetune} &9.3 &9.5 & 10.4 & 9.7 & 29.9 & 30.5 & 33.6 & 31.3    \\ 
\multicolumn{2}{c|}{Prompt Tuning$^{\ast}$} & 9.7 & 24.4 & 12.2 & 17.4 & - & -& -& -   \\ 
\multicolumn{2}{c|}{Data Replay} &46.0 & 50.3 & 34.6 & 43.6 & 34.9 & 39.3& 34.9& 36.4   \\ 
\multicolumn{2}{c|}{LFPT5$^{\ast}$} &54.7 & 54.1 & 54.2 & 54.3 & - & -& -& -   \\ 
\multicolumn{2}{c|}{Per-task Prompt$^{\ast}$} & 69.9 &69.9 &69.9 & 69.8 & 50.6 & 50.6& 50.6& 50.6\\
\multicolumn{2}{c|}{IDBR} & - & - & - & - & 39.7 & 37.9& 32.9& 36.8   \\ 
\multicolumn{2}{c|}{ProgPrompt$^{\ast}$} &75.4 &76.6 &76.7 & 76.2 & 55.3 & 53.3& {51.9} & 53.5   \\ 
\midrule
\multicolumn{1}{c|}{\multirow{3}{*}{\makecell{Ours$^{\ast}$\\($\Qc_{\text{size}}=5$)}}} & Random &76.4 &77.3 &76.1 & 76.6 & 53.6& 53.2 & 51.1 & 52.6 \\
\multicolumn{1}{c|}{}& FIFO & 76.5& 77.2& 76.7&   76.8 & 54.5& 53.8 & 51.8 & 53.4  \\
\multicolumn{1}{c|}{}& DQ-PCA & {77.5}& {78.8}& {77.8} & {78.0} & {55.6}&{56.0} &51.8 & {54.5} \\
\midrule
\multicolumn{1}{c|}{\multirow{3}{*}{\makecell{Ours$^{\ast}$\\($\Qc_{\text{size}}=10$)}}} & Random & 76.7 & 77.2 &76.5 & 76.8 & 54.7& 54.2 & 52.8 & 53.9 \\
\multicolumn{1}{c|}{} & FIFO & 77.0& 77.1& 76.7&   76.9 & 54.6& 54.2 & \textbf{52.9} & 53.9  \\
\multicolumn{1}{c|}{} & DQ-PCA & \textbf{78.3}& \textbf{79.7}& \textbf{78.7} & \textbf{78.9} & \textbf{56.5}& \textbf{56.2} & 52.6 & \textbf{55.1}  \\
\midrule
\multicolumn{2}{c|}{\multirow{1}{*}{\makecell{ProgPrompt + Aggregation +  $\theta_{\Pc^{\ast}}$}}} & 79.0& 79.1& 78.1 & 78.7 & 55.3& 55.2 & 54.5 & 55.0 
\\
\multicolumn{2}{c|}{MTL} &70.7 &70.7 &70.7 & 70.7 & 56.9 & 56.9& 56.9& 56.9   \\ 
\bottomrule
\end{tabular}
}
\captionof{table}{Average test set performance of Q-tuning and prior approaches on long-sequence experiments with 15 text classification tasks in different orders. In the experiments\protect\footnotemark, we use the few-shot CL by setting 20 samples per class. 
All the results are averaged over 3 runs.}
\label{tab:long_seq}
\end{table*}


\subsection{Results on Few-shot CL Benchmarks}
\paragraph{Short-sequence Experiments:} 
Following ProgPrompt \cite{razdaibiedina2023progressive}, we evaluate the performance of Q-tuning on the standard short-sequence CL benchmarks with few-shot learning settings, where Orders 1$\sim$3 and Orders 4$\sim$7 are evaluated with the T5 and BERT models, respectively.  
Since these sequential tasks only consist of four or five disjoint datasets, we set $\Qc_{\text{size}}=5$ for the Q-prompt without utilizing the DQ-PCA rule. In Table~\ref{tab:table_t5}, we compare Q-tuning with the existing CL, prompt tuning, and continual prompt tuning approaches using the T5 model. Q-tuning outperforms all the CL approaches by a large margin, achieving 76.2\% accuracy on average of all the orders. 
Q-tuning increases the accuracy by 1.7\% (from 74.5\% to 76.2\%) compared to ProgPrompt, the SOTA approach of continual prompt tuning. Q-tuning also surpasses the ``Per-task Fintune'' by 6.2\% on average, demonstrating the efficacy of the proposed queue aggregation and shared prefix prompt approach in enhancing the FKT capability. Table~\ref{tab:table_bert} reports the results on the BERT-base model that verify a consistent improvement.

\paragraph{Long-sequence Experiments:}
In Table \ref{tab:long_seq}, we compare the Q-tuning with the baseline approaches on the long-sequence CL benchmark, including Orders 8$\sim$10 using the T5-large and the BERT-base models. 
These experiments consist of 15 tasks in three different orders. We follow the few-shot CL setting as in  \citet{razdaibiedina2023progressive} by selecting 20 samples per class. The row of ``ProgPrompt + $\theta_{\Pc^{\ast}}$'' denotes the results by adding the shared prompt to ProgPrompt, while maintaining its complete prompt list of 15 prompts. We can observe that setting the maximum length of the Q-prompt to 5 using DQ-PCA only leads to a 0.7\% accuracy drop (from 78.7\% to 78.0\%) compared with the ``ProgPrompt + Aggregation +  $\theta_{\Pc^{\ast}}$''. Moreover, we even observe a 0.2\% accuracy increase over the full prompt when setting the maximum Q-prompt length to 10. 
This indicates the capability of DQ-PCA to protect essential knowledge when trimming the Q-prompt. 

\footnotetext{MTL denotes multi-task learning that fintunes the model using all the datasets from different tasks. 
Methods marked with $\ast$ only train a soft prompt while freezing the pretrained model, other methods train the entire model. The ``Full Prompts'' denotes remaining all prompts in queue by setting $\Qc_{\text{size}}=15$.}

In addition, in Table \ref{tab:long_seq}, we compare the proposed DQ-PCA with the two naive queue eviction rules including random dropping, first in and first out (FIFO). As the results suggest, evicting Q-prompt by random dropping and FIFO yields very close performance, which both blindly shrink the Q-prompt without considering the relevance across different promtops. Unlike them, our DQ-PCA shrinks the Q-prompt by preserving the most informative prompts, thus clearly outperforming those two naive strategies. The results on the T5-large and the BERT-base models collectively witness the superiority of using our DQ-PCA.


\begin{table}[htbp]
\centering
\fontsize{8}{0}\selectfont
\small
\scalebox{0.85}{
\begin{tabular}{lccccc}
\toprule
 \multicolumn{2}{c}{\multirow{2}{*}{\textbf{Method}}}&\multicolumn{3}{c}{\multirow{1}{*}{\textbf{T5-large}}} & \\
&&Order 11& Order 12& Order 13 & \multirow{1}{*}{\textbf{Average}}\\
\midrule
\multicolumn{2}{c}{ProgPrompt\protect\footnotemark} & Fail& Fail & Fail &-\\
\midrule
  \multicolumn{2}{c}{Per-task Prompt} & 60.4 & 60.4 & 60.4 & 60.4\\
  \multicolumn{2}{c}{Shared Prompt} & 62.4 & 62.7 & 63.1 & 62.7\\
 \midrule
 \multicolumn{2}{c}{\makecell{Q-tuning\\($\Qc_{\text{size}}=10$)}}& \textbf{90.9} &  \textbf{90.6}& \textbf{90.8} & \textbf{90.8}\\
\bottomrule
\end{tabular}
}
\captionof{table}{\small{Results on extremely long sequence experiments (70 tasks). All results are averaged over 3 runs.}}
\label{tab:extreme_long_exp}
\end{table}


\paragraph{Lifelong Learning Experiments:}
Lastly, we use extremely long task sequences to mimic lifelong learning scenarios. 
Table~\ref{tab:extreme_long_exp} reports the results of Q-tuning on Orders 11$\sim$13 including three {\em random} permutations of 70 {\em disjoint} tasks. 
Training ProgPrompt will fail due to out of memory caused by the accumulation of prompts\footref{progfail_footnote}\footnotetext{\label{progfail_footnote}When using the same batch size as Q-tuning, ProgPrompt encounters failure during training after the $15$-th task.}. 
Compared to the per-task prompt tuning, Q-tuning has gained considerable performance benefits ({30.4\%} accuracy improvement on average from 60.4\% to 90.8\%). 
This can be attributed to 1) the improved FKT by applying Q-prompt knowledge aggregation, 2) the effective trimming of Q-prompt using DQ-PCA to enable the training of long sequence of tasks, and 3) the use of MR regularization and shared prefix prompt to avoid the accumulated information loss caused by the Q-prompt trimming. 
We also compare Q-tuning with training using a global shared prompt and a per-task prompt plus the MR regularization for each task without maintaining a queue. 
To ensure a fair comparison, we set the length of the shared prompt to be identical to Q-tuning, \ie $l\times\Qc_{\text{size}}$. Although the accuracy of the shared prompt is better than the per-task prompt tuning ({2.3\%} improvement on average from {60.4\%} to {62.7\%}), it is outperformed by Q-tuning by {28.1\%} ({62.7\%} to {90.8\%}) on average. 
This indicates that, although the Q-prompt and the shared prefix prompt serve the same purpose of aggregating knowledge for better FKT, it is beneficial to keep both components.

\begin{figure}[htbp]
\fontsize{9}{6}\selectfont
\begin{minipage}{.48\textwidth}
\centering
\setlength{\tabcolsep}{3pt}
\fontsize{10}{6}\selectfont
\small
\scalebox{0.8}{
\begin{tabular}{ccccc}
\toprule
\thead{Forward Transfer\\(Target Task)} & \thead{Q-prompt\\ (Full)} & \thead{Q-prompt \\ ($\Qc_{\text{size}}=5$)} & \thead{Q-prompt \\($\Qc_{\text{size}}=10$)} & \thead{ Prompt Tuning}  \\
\midrule
 Task 11 & 98.1 & 97.8 ({\bf \tiny{$\downarrow$ 0.3\%}}) & 98.2 ({\bf \tiny{$\uparrow$ 0.1\% }})& 97.1 ({\bf \tiny{$\downarrow$ 1.0\%}}) \\
Task 12 &86.2& 83.9 ({\bf \tiny{$\downarrow$ 2.3\%}}) & 86.1 ({\bf \tiny{ $\downarrow$ 0.1\%}})& 72.6 ({\bf \tiny{$\downarrow$ 13.6\%}})\\
Task 13 & 56.6 & 54.9 ({\bf \tiny{$\downarrow$ 1.7\%}}) & 56.2 ({\bf \tiny{$\downarrow$ 0.4\%}}) & 49.8 ({\bf \tiny{$\downarrow$ 6.8\%}})\\
Task 14 & 50.4& 50.3 ({\bf \tiny{$\downarrow$ 0.1\%}})& 50.5 ({\bf \tiny{$\uparrow$ 0.1\%}})& 47.6 ({\bf \tiny{$\downarrow$ 2.8\%}})\\
Task 15& 69.4 & 68.9 ({\bf \tiny{$\downarrow$ 0.5\%}}) & 69.1 ({\bf \tiny{$\downarrow$ 0.3\%}})&68.1 ({\bf \tiny{$\downarrow$ 1.3\%}})\\
\midrule
Average & 72.1 & 71.2 ({\bf \tiny{$\downarrow$ 0.9\% }}) & \textbf{72.0} ({\bf \tiny{$\downarrow$ 0.1\%}})&67.0 ({\bf \tiny{$\downarrow$ 5.1\%}}) \\
\bottomrule
\end{tabular}
}
\captionof{table}{\small{Forward knowledge transfer results of Order 9 using 20 samples/class. Results are averaged over 3 runs.}}
\label{table:transfer_table}
\end{minipage}
\begin{minipage}{0.48\textwidth}
\centering
\vspace{7pt}
\setlength{\tabcolsep}{5pt}
\fontsize{8}{6}\selectfont
\small
\scalebox{0.75}{
\begin{tabular}{c|ccc|cccc}
\toprule
 \multirow{2}{*}{\textbf{Sequence}}&\multicolumn{3}{c|}{\multirow{1}{*}{\textbf{Method}}}&\multicolumn{3}{c}{{\multirow{1}{*}{\textbf{Num. samples}}}}&\\
 &Q-prompt & Aggregation & $\theta_{\Pc^{\ast}}$ & 16 & 200 & 1000 & \textbf{Average}\\
\midrule
\multirow{3}{*}{Short}&\cmark &\xmark &\xmark & 74.5 & 79.8& 79.8& 78.0  \\ 
&\cmark& \cmark&\xmark  & 75.2 & 80.9& 80.4& 78.8   \\
&\cmark& \xmark& \cmark & 75.1 & 80.6& 80.9 & 78.9  \\
&\cmark& \cmark& \cmark&  \textbf{76.2} & \textbf{81.2}& \textbf{81.9} & \textbf{79.7}  \\
\midrule
 \multirow{2}{*}{\textbf{Sequence}}&\multicolumn{3}{c|}{\multirow{1}{*}{\textbf{Method}}}&\multicolumn{3}{c}{{\multirow{1}{*}{\textbf{Num. samples}}}}&\\
 &Q-prompt & Aggregation & $\theta_{\Pc^{\ast}}$ & 20 & 200 & 1000 & \textbf{Average}\\
\midrule
\multirow{4}{*}{Long}&\cmark &\xmark &\xmark & 76.7 & 80.8& 80.8 & 79.4 \\ 
&\cmark& \cmark&\xmark  & 77.2 & 81.1& 82.1 & 80.2  \\
&\cmark&\xmark & \cmark & 77.4 & 81.1& 82.3  & 80.3 \\
&\cmark& \cmark& \cmark& \textbf{78.9} & \textbf{81.9}& \textbf{83.3} & \textbf{81.4} \\
\bottomrule
\end{tabular}
}
\captionof{table}{\small{Ablation studies on the Q-prompt aggregation and shared prefix prompt\protect\footnotemark. Results are averaged over 3 runs.}}
\label{tab:abaltionstud_average}
\end{minipage}
\end{figure}
\footnotetext{For long sequence, we set $\Qc_{\text{size}} = 10$. More detailed results of each order are reported in Appendix~\ref{appendix:more_results}.}

\subsection{Ablation Study and Analysis}
In this section, we evaluate our approach's performances in various aspects, including its capability of fulfilling FKT, adapting previous prompts based on their relevance to the new task using the Q-prompt aggregation, and maintaining global knowledge sharing using a shared prefix prompt.  

\textbf{Forward Knowledge Transfer:}\, In Table~\ref{table:transfer_table}, we evaluate the FKT performance of the trimmed Q-prompt. We train three different Q-prompts including the ``Full'', ``$\Qc_{\text{size}}=5$'' and ``$\Qc_{\text{size}}=10$'', where the ``Full'' denotes keeping the complete Q-prompt without the De-Q operation. All these Q-prompts are continuously trained on the first 10 tasks of Order 9. Then we separately evaluate the FKT performance of these Q-prompts on five remaining target tasks. 
As a reference, we also train a single prompt (denoted by ``Prompt Tuning'' whose token length is set the same as the total length of the full Q-prompt) on each target task. First of all, full Q-prompt substantially outperforms ``Prompt Tuning'', demonstrating our approach's capability in fulfilling FKT whereas ``Prompt Tuning'' does not leverage any information from other tasks.
Moreover, compared to the full Q-prompt, the trimmed Q-prompt only has a minor performance drop. For example, setting $\Qc_{\text{size}}=10$ only leads to 0.1\% accuracy decrease (from 72.1\% to 72.0\%). This proves that trimmed Q-prompt is able to maintain FKT at the same level as the full Q-prompt, despite previous prompts being trimmed.

\begin{table}[hbpt]
\centering
\begin{minipage}[t]{1\linewidth}
\centering
\includegraphics[width=0.9\textwidth]{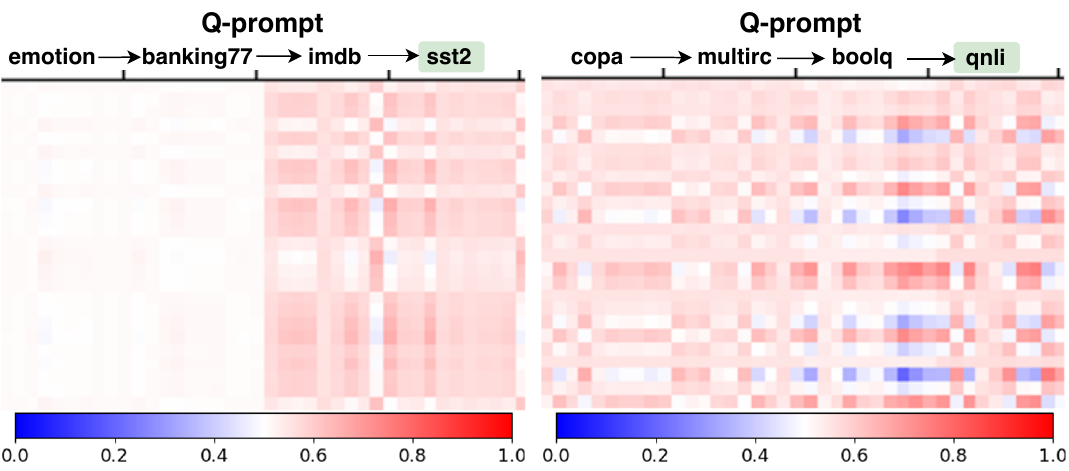}
\captionof{figure}{\footnotesize Visualization of aggregation matrix. }
\label{fig:weight_matrix}
\end{minipage}
\\
\vspace{0.3cm}
\begin{minipage}[t]{1\linewidth}
\centering
\setlength{\tabcolsep}{2pt}
\fontsize{8}{0}\selectfont
\small
\scalebox{0.9}{
\begin{tabular}{ccccccc}
\toprule
 \multicolumn{3}{c}{\multirow{1}{*}{\textbf{Method}}}&\multicolumn{3}{c}{\multirow{1}{*}{\textbf{T5-large}}} & \\
$\theta_{\Pc^{\ast}}$ & Aggregation & $\Lc_{\text{MR}}$&Order 11& Order 12& Order 13 & \multirow{1}{*}{\textbf{Average}}\\
\midrule
 \xmark&\xmark &\xmark&86.8 & 87.3 & 87.7 & 87.3 \\
 \cmark&\xmark &\xmark & 87.6 & 88.1 & 88.7 & 88.1\\
 \cmark&\cmark &\xmark & 89.8 & 89.4&90.1& 89.8\\
 \cmark&\cmark &\cmark& \textbf{90.9} &  \textbf{90.6}& \textbf{90.8} & \textbf{90.8}\\
\bottomrule
\end{tabular}
}
\captionof{table}{Ablation studies on the extremely long sequence experiments. Results are averaged over 3 runs.}
\label{tab:ablation_extreme_long_exp}
\end{minipage}
\end{table}


\textbf{Q-prompt Aggregation:}\, Table~\ref{tab:abaltionstud_average} demonstrates the efficacy of the knowledge aggregation technique. In both the short and long task sequences, compared with the complete Q-prompt model (the fourth row), dropping the knowledge aggregation (the third row) leads to 0.8\% and 1.1\% accuracy drop in the short and long task sequences, respectively. 
In addition, in Fig.~\ref{fig:weight_matrix}, we visualize the trained weight matrix $\Wc$ to reflect the relevance of previously learned prompts to the current task. 
We can observe when learning the ``sst2'' task, the prompt from the ``imdb'' task contributes the most. This is because the two tasks are both for the movie review classification. The aggregation matrix uncovers their correlation and assigns more weights to the prompt of the ``imdb'' task. 
In contrast, for the ``qnli'' task, the aggregation matrix suggests an even contribution of each prompt in the queue. This is because all the tasks are of the Q\&A classification. 

\textbf{Shared Prefix Prompt and MR:}\, We conduct ablation studies to validate the efficacy of the shared prefix prompt and the MR regularization. As shown in Table~\ref{tab:abaltionstud_average}, 
by comparing the complete Q-prompt (the fourth row) and dropping the shared prefix prompt (the second row), we observe an accuracy drop of 0.9\% and 1.2\% in the short and long task sequences, respectively. This negative impact in the short sequence is weaker than that of the long task sequence. This is expected as the short task sequence does not utilize DQ-PCA to trim the Q-prompt, hence no information loss, which dilutes the effect of the shared prefix prompt. Furthermore, to evaluate the contribution of the MR regularization, we conduct the experiments on a long task sequence by setting $\mathcal{Q}_{\text{size}}=10$. Table~\ref{tab:ablation_extreme_long_exp} shows that dropping the MR regularization of the shared prefix prompt (from the third row to the second row) leads to a 1\% accuracy drop. 
In Appendix~\ref{appendix:more_results}, we report the performance using different values of $\eta$ for the MR regularization. 

\section{Conclusion}\label{Section:conclusion}
This paper introduces a new model-agnostic approach named Q-tuning, which can pave the way to achieving lifelong continual prompt tuning for present and future LMs with a rapid growth of parameters.  
In comparison with existing CL methods, Q-tuning maintains a low-cost prompt queue instead of storing a large number of task-specific parameters or saving old data samples for replay. 
Our extensive experiments demonstrate that Q-tuning outperforms existing continual learning, prompt tuning and continual prompt tuning methods on the standard CL benchmarks for text classification. 
In addition, we verify the effectiveness of Q-tuning on both short and long task sequences, including up to 70 tasks that mimic the case of lifelong learning. 

\textbf{Limitations:}\, Although Q-tuning demonstrates a strong FKT capability, it does not enable the backward knowledge transfer as both the model and the previous Q-prompts are frozen during the learning of a new task. 
Besides, Q-tuning requires the task identities to be known at test time. To address the more challenging CL scenario when the task identities are undisclosed at test time, for task $i$, we can assign a trainable query key $k^i$ to the corresponding Q-prompt $\Qc^{i}$ and jointly train $k^i$ to maximize the similarity between $k^i$ and the feature of each sample $x$ from task $i$. During test time, given an input $x^{\prime}$ with an unknown identity, we will first locate the Q-prompt that has the largest similarity between its key $k^j$ and the input $x^{\prime}$, and then we can use the retrieved Q-prompt $\Qc^{j}$ to infer $x^{\prime}$. We will address this problem in our future work. 



%

\bibliography{Xin_Async}

\appendix
\clearpage
{\large\textbf{Appendix}}

\section{Q-tuning Algorithm
\label{appendix:q_tuing_algorithm}}
\begin{algorithm}[htbp]
\fontsize{10}{6}\selectfont
\caption{Q-tuning Algorithm} 
\begin{algorithmic}
	\Input{Continually increased task set $\Tc$, Q-prompt $\Qc$ with a maximum capacity $C$, fixed pretrained model $\theta_{\Mcc}$, aggregation matrix $\Wc$ for $\Qc$, shared prefix prompt $\theta_{\Pc^{\ast}}$, memory factor $\eta$.}
    \Initialize{$ \Qc^{1}=\{\}$, randomly initialized $\theta^1_{\Pc^{\ast}}$ and $\theta^1_{\Pc}$, initialized $\Wc^1$ with an identity matrix.}
	\For {continually coming task $i=1,2,\ldots$}
    \If {$i>C$}
	\State $\Qc \leftarrow \text{PCA-DQ}(\Qc)$ // De-Q (Eq.(\ref{eq_enqueue}))
	\Else 
	\State $\Qc \leftarrow \Qc \oplus \theta^i_{\Pc}$ // En-Q 
	\EndIf 
	\For {each batch sample from $\Tc^i$'s dataset}
	 \State $\theta^i_{\Pc} \leftarrow \theta^i_{\Pc} + \nabla_{\theta^i_{\Pc}} \Lc_{\Qc}(\theta^i_{\Pc^{\ast}},\theta^i_{\Pc},\Wc^i)$
	  \State $\Wc^{i} \leftarrow  \Wc^{i} + \nabla_{\Wc^{i}} \Lc_{\Qc}(\theta^i_{\Pc^{\ast}},\theta^i_{\Pc},\Wc^i)$

	\If {i=1}
	 	 \State $\theta^i_{\Pc^{\ast}} \leftarrow \theta^i_{\Pc^{\ast}} + \nabla_{\theta^i_{\Pc^{\ast}}} \Lc_{\Qc}(\theta^i_{\Pc^{\ast}},\theta^i_{\Pc},\Wc^i)$
	 \ElsIf{$i>C$} 
    	     \State $\theta^i_{\Pc^{\ast}} \leftarrow \theta^i_{\Pc^{\ast}} + \nabla_{\theta^i_{\Pc^{\ast}}}[ \Lc_{\Qc}(\theta^i_{\Pc^{\ast}},\theta^i_{\Pc},\Wc^i) + \eta \Lc_{\text{MR}}(\theta^i_{\Pc^{\ast}})]$
	\EndIf
    \EndFor
\EndFor
\end{algorithmic} 
\label{q_tuning_algorithm}
\end{algorithm}

\section{Mutual Information Estimation
\label{appendix:mutual_information}}

\begin{proposition}
  Let $p(x)$ and $p(y)$ represent two random variables, their mutual information satisfies 
  \begin{equation}
  {\fontsize{10}{0}\selectfont
    \begin{aligned}
        I(&p(x);p(y)):=\mathcal{D}_{\mathrm{JSD}}(\mathbf{J}||\mathbf{M})\\
        &\geq \mathbf{E}_{z\sim \mathbf{J}}\left[-\sigma(-\mathcal{F}_{\omega}(z))\right] - \mathbf{E}_{z^{\prime}\sim \mathbf{M}}\left[\sigma(\mathcal{F}_{\omega}(z^{\prime}))\right]
    \end{aligned}
}
  \label{jsd_divergence}
  \end{equation}
  where the joint $\mathbf{J} = p(x,y)$, $\mathbf{M} = p(x)p(y)$ is the product of the marginals,
  $\sigma(t)=\mathrm{log}(1+e^t)$, and $\mathcal{F}$ belongs to an arbitrary class of functions that can map $\mathbf{J}\rightarrow \mathbb{R}$ and $\mathbf{M}\rightarrow \mathbb{R}$.
\end{proposition}

\begin{proof}
  According to the variational estimation of $f$-divergences \citep{Nguyen2010}, we have 
  \begin{equation}
  {\fontsize{9}{0}\selectfont
    \begin{aligned}
      \mathcal{D}_f(\mathbf{P}||\mathbf{Q})&=\int q(x)\sup\limits_{t\in \mathrm{dom}_{g^*}}{t \frac{p(x)}{q(x)}-g^*(t)} \mathrm{d}x\\
      &\geq \sup\limits_{\mathcal{V}\in F}\biggl(\int p(x)\mathcal{V}(x)\mathrm{d}x-\int q(x)g^*(\mathcal{V}(x))\mathrm{d}x\biggr)\\
      &=\sup\limits_{\mathcal{V}\in F}(\mathbb{E}_{x\sim \mathbf{P}}[\mathcal{V}(x)]-\mathbb{E}_{x\sim \mathbf{Q}}[g^*(\mathcal{V}(x))])
    \end{aligned}
 }
    \label{f_divergence_def}
  \end{equation}
  where the function $g^*$ is a convex conjugate function \citep{hiriart2004fundamentals,nowozin2016f} of a convex, lower-semicontinuous function. The function $g^*$ is defined as 
  \begin{equation}
    \begin{aligned}
      g^*(t) = \sup_{u\in \mathrm{dom}_f}\{ut-f(u)\}
    \end{aligned}
  \end{equation}
  We parameterize $\mathcal{V}$ using a neural network with parameters $w$ and write it as $\mathcal{V}_{\omega}$. 
  We assume the form of the function $\mathcal{V}_{\omega}=g_f(\mathcal{F}_{\omega}(x))$. 
  Given two probability distributions $\mathbf{J} = p(x,y)$ and $\mathbf{M} = p(x)p(y)$, their $f$-divergence satisfies:
  \begin{equation}
  {\fontsize{10}{0}\selectfont
    \begin{aligned}
      \mathcal{D}_f(\mathbf{J}||\mathbf{M}) = & \sup\limits_{\mathcal{F}_{\omega}}(\mathbb{E}_{z\sim \mathbf{J}}[g_f(\mathcal{F}_{\omega}(z))] \\
      &-\mathbb{E}_{z^{\prime}\sim \mathbf{M}}[g^*(g_f(\mathcal{F}_{\omega}(z^{\prime})))])
    \end{aligned}
}
    \label{paramized_f_divergence}
  \end{equation}
  where $g_f$ is an activation function specific to the $f$-divergence used. 
  Table~\ref{activation_tab} provides the commonly used $g_f$ and the convex conjugate function $g^*$. 
  According to this table, for the $\mathrm{JSD}$ based divergence, we have $g_f(x)=\mathrm{log}(2)-\mathrm{log}(1+\mathrm{exp}(-x))$ 
  and $g^*(x)=-\mathrm{log}(2-\mathrm{exp}(x))$. By substituting them into Eq.~(\ref{paramized_f_divergence}), 
  we have 
  \begin{equation}
  {\fontsize{10}{0}\selectfont
    \begin{aligned}
      \mathbb{E}_{z\sim \mathbf{J}}\left[g_f(\mathcal{F}_{\omega}(z))\right] &= \mathbb{E}\left[\mathrm{log}2 - \mathrm{log}(1+\mathrm{exp}(-\mathcal{F}_{\omega}(z)))\right]\\
      &= \mathbb{E}_{z\sim \mathbf{J}}\left[\mathrm{log}2-\sigma(-\mathcal{F}_{\omega}(z))\right]
    \end{aligned}
}
    \label{f_divergence_p1}
  \end{equation}
  \begin{equation}
  {\fontsize{9}{0}\selectfont
    \begin{aligned}
      \mathbb{E}_{z^{\prime}\sim \mathbf{M}}&\left[g^*(g_f(\mathcal{F}_{\omega}(z^{\prime})))\right] \\
      &= \mathbb{E}_{z^{\prime}\sim \mathbf{M}}\left[-\mathrm{log}(2-\mathrm{exp}^{\mathrm{log}2-\mathrm{log}(1+\mathrm{exp}(-\mathcal{F}_{\omega}(z^{\prime})))})\right]\\
      &= \mathbb{E}_{z^{\prime}\sim \mathbf{M}}\left[-\mathrm{log}(2-2(1+\mathrm{exp}(-\mathcal{F}_{\omega}(z^{\prime}))^{-1}))\right]\\
      &= \mathbb{E}_{z^{\prime}\sim \mathbf{M}}\left[-\mathrm{log}(2\frac{\mathrm{exp}(-\mathcal{F}_{\omega}(z^{\prime}))}{1+\mathrm{exp(-\mathcal{F}_{\omega}(z^{\prime}))}})\right]\\
      &= \mathbb{E}_{z^{\prime}\sim \mathbf{M}}\left[-\mathrm{log}(\frac{2}{\mathrm{exp}(\mathcal{F}_{\omega}(z^{\prime}))+1})\right]\\
      &= \mathbb{E}_{z^{\prime}\sim \mathbf{M}}\left[-(\mathrm{log}2 - \mathrm{log}(\mathrm{exp}(\mathcal{F}_{\omega}(z^{\prime}))+1))\right]\\
      &= \mathbb{E}_{z^{\prime}\sim \mathbf{M}}\left[-\mathrm{log}2 + \sigma(\mathcal{F}_{\omega}(z^{\prime}))\right]
    \end{aligned}
    }
    \label{f_divergence_p2}
  \end{equation}
Combining Eq.~(\ref{f_divergence_p1}) and Eq.~(\ref{f_divergence_p2}), we can rewrite Eq.~(\ref{paramized_f_divergence}) as a $\mathrm{JSD}$-divergence based form:  
\begin{equation}
{\fontsize{9}{0}\selectfont
  \begin{aligned}
    \mathcal{D}_{\mathrm{JSD}}(\mathbf{J}||\mathbf{M}) = & \sup\limits_{\mathcal{F}_{\omega}}(\mathbb{E}_{z\sim \mathbf{J}}\left[\mathrm{log}2\right]+ \mathbb{E}_{z\sim \mathbf{J}}\left[-\sigma(-\mathcal{F}_{\omega}(z))\right]\\
  &+\mathbb{E}_{z^{\prime}\sim \mathbf{M}}\left[\mathrm{log}2\right] -\mathbb{E}_{z^{\prime}\sim \mathbf{M}}\left[\sigma(\mathcal{F}_{\omega}(z^{\prime}))\right])\\
  \geq &\mathbb{E}_{z\sim \mathbf{J}}\left[-\sigma(-\mathcal{F}_{\omega}(z))\right]-\mathbb{E}_{z^{\prime}\sim \mathbf{M}}\left[ \sigma(\mathcal{F}_{\omega}(z^{\prime}))\right]
  \end{aligned}
}
\label{paramized_jsd_divergence}
\end{equation}
\end{proof}

\begin{table*}[htbp]
  \centering
  \setlength{\tabcolsep}{2pt}
    \fontsize{14}{0}\selectfont
    \small
  \begin{tabular}{cccccc}
    \toprule
    Name & Output activation $g_f$ & $\mathbf{dom}_{g^\star}$ & Conjugate $g^{\star}(t)$  \\
    \midrule
    Kullback-Leibler (KL) & $v$ & $\mathbb{R}$ & $\mathrm{exp}(t-1)$ \\
    Reverse KL & $-\mathrm{exp}(-v)$ & $\mathbb{R}_-$ & -1-$\mathrm{log}(-t)$ \\
    Pearson $\chi^2$  & $v$ & $\mathbb{R}$ & $\frac{1}{4}t^2+t$ \\
    Square Hellinger & $1-\mathrm{exp}(-v)$ & $t<1$ & $\frac{t}{1-t}$ \\
    Jensen-Shannon & $\mathrm{log}(2)-\mathrm{log}(1+\mathrm{exp}(-v))$ & $t<\mathrm{log}(2)$ & $-\mathrm{log}(2-\mathrm{exp}(t))$\\
    \bottomrule
  \end{tabular}
\caption{Recommended final layer activation functions and their conjugate functions. This table comes from \citet{nowozin2016f}.}
\label{activation_tab}
\end{table*}

\section{More Implementation Details}

\subsection{Datasets and Metrics}
\label{appendix:datasets}
We use 21 public datasets, of which 15 datesets are the same as ProgPrompt \cite{razdaibiedina2023progressive} for our experiments. 
Table~\ref{tab:datasets_table} reports the details of the 21 datasets, along with their evaluation metrics. 
Overall, we use datasets from CL benchmark \citep{zhang2015character}, GLUE \citep{wang2018glue} and SuperGLUE \citep{wang2019superglue} benchmarks, and IMDB movie reviews dataset.
We use the Banking77 dataset \citep{casanueva2020efficient} and Emotion dataset \citep{saravia2018carer} for the extremely long 70-task experiments. Following the common practice, for tasks that have two evaluation metrics, we use the average of the two as the final performance metric. 

To mimic the life-long learning, we add WNLI, COLA, and QNLI from the GLUE benchmark, WSC from the SuperGLUE benchmark, the Banking77 dataset \citep{casanueva2020efficient} and the Emotion dataset  \citep{saravia2018carer} to form an extremely long sequence including 70 tasks. In the 70-task experiments, we split the DBpedia set into 7 \textbf{disjoint} tasks, the Yahoo set into 5 \textbf{disjoint} tasks, and the Banking77 set into 38 \textbf{disjoint} tasks (removing 1 class), and the Emotion dataset into 3 \textbf{disjoint} tasks, where each task has two 2 classes. These divided 53 subsets plus the rest 17 datasets form the final 70-task dataset. 
Following \cite{razdaibiedina2023progressive}, for each task, we randomly select 500 samples per class from the training set for validation, and use early stopping according to the validation accuracy on all seen tasks.

\begin{table*}[htbp]
\centering
\fontsize{14}{0}
\small
\captionof{table}{The details of 21 datasets used in our experiments. NLI denotes natural language inference, QA denotes questions and answers task, and EM denotes exact match scoring. The first five tasks are used to form the standard CL benchmark, all other tasks are used in our long-sequence experiments.}
\scalebox{1}{
\begin{tabular}{l|llllc}
\toprule
 \textbf{Dataset name} & \textbf{Category} & \textbf{Task} &  \textbf{Domain} & \textbf{Metric} & \textbf{Classes} \\
 \midrule
 1. YP & CL benchmark & sentiment analysis & YP reviews & accuracy & 5 \\
 2. Amazon & CL benchmark & sentiment analysis & Amazon reviews & accuracy & 5 \\
 3. DBpedia & CL benchmark & topic classification & Wikipedia & accuracy &14\\
 4. Yahoo & CL benchmark & QA & Yahoo Q\&A & accuracy&10 \\
 5. AG News & CL benchmark & topic classification & news & accuracy & 4 \\
 6. MNLI & GLUE & NLI & various & accuracy & 3\\
 7. QQP & GLUE & paraphrase detection & Quora & accuracy \& F1 & 2 \\
 8. RTE & GLUE & NLI & news, Wikipedia & accuracy & 2 \\
 9. SST2 & GLUE & sentiment analysis & movie reviews & accuracy & 2\\
 10. WiC & SuperGLUE & word sense disambiguation & lexical databases & accuracy & 2\\
 11. CB & SuperGLUE & NLI & various & accuracy & 2 \\
 12. COPA & SuperGLUE & QA & blogs, encyclopedia & accuracy & 2 \\
 13. BoolQ & SuperGLUE & boolean QA & Wikipedia & accuracy & 2\\
 14. MultiRC & SuperGLUE & QA & various & F1 \& EM & 2\\
 15. IMDB & Other & sentiment analysis & movie reviews & accuracy & 2 \\
 16. WNLI & GLUE & NLI & various & accuracy & 2 \\
 17. COLA & GLUE & NLI & books, journal articles & accuracy & 2 \\
 18. QNLI & GLUE & QA & Wikipedia & accuracy & 2 \\
 19. WSC & SuperGLUE & NLI & various & accuracy & 2 \\
 20. Banking77 & Other & intent detection & banking& accuracy & 77 \\
 21. Emotion & Other & emotion detection & Twitter& accuracy & 6 \\
 \bottomrule
\end{tabular}
}
\label{tab:datasets_table}
\end{table*}

\subsection{Task sequence orders}
We report the task orders used in our experiments across the T5 and BERT models in Table~\ref{tab:seq} below, where Orders 1-10 are the same as ProgPrompt \citep{razdaibiedina2023progressive}. The Orders 11-13 are created by \textbf{randomly permuting} the collected 70 disjoint datasets to mimic the lifelong learning of continuously incoming unseen tasks. 

\label{appendix:orders}

\newcommand{\arr}[1][3pt]{\mathrel{%
   \vcenter{\hbox{\rule[-.2pt]{#1}{.4pt}}}%
   \mkern-5mu\hbox{\usefont{U}{lasy}{m}{n}\symbol{41}}}}

\begin{table*}[htbp]
\fontsize{9pt}{9pt}\selectfont
\tabcolsep=2pt
\begin{tabular}{lll}
\\
\toprule
\textbf{Order} & \textbf{Model} & \textbf{Task Sequence} \\
\midrule
1 & T5 & db $\arr$ amazon $\arr$ yahoo $\arr$ ag\\
2 & T5 & db $\arr$ amazon $\arr$ ag $\arr$ yahoo\\
3 & T5 & yahoo $\arr$ amazon $\arr$ ag $\arr$ db \\
\midrule
4 & BERT & ag $\arr$ yp $\arr$ amazon $\arr$ yahoo $\arr$ db \\
5 & BERT & yp $\arr$ yahoo $\arr$ amazon $\arr$ db $\arr$ ag \\
6 & BERT & db $\arr$ yahoo $\arr$ ag $\arr$ amazon $\arr$ yp \\
7 & BERT & yp $\arr$ ag $\arr$ db $\arr$ amazon $\arr$ yahoo \\
\midrule
8 & T5, BERT & \makecell{mnli $\arr$ cb $\arr$ wic $\arr$ copa $\arr$ qqp $\arr$ boolq $\arr$ rte $\arr$ imdb $\arr$ \\ yp $\arr$ amazon $\arr$ sst2 $\arr$ dbpedia $\arr$ ag $\arr$ multirc $\arr$ yahoo} \\[2ex]
9 & T5, BERT & \makecell{multirc $\arr$ boolq $\arr$ wic $\arr$ mnli $\arr$ cb $\arr$ copa $\arr$ qqp $\arr$ rte $\arr$ \\ imdb $\arr$ sst2 $\arr$ dbpedia $\arr$ ag $\arr$ yp $\arr$ amazon $\arr$ yahoo} \\[2ex]
10 & T5, BERT & \makecell{yp $\arr$ amazon $\arr$ mnli $\arr$ cb $\arr$ copa $\arr$ qqp $\arr$ rte $\arr$ imdb $\arr$  \\ sst2 $\arr$ dbpedia $\arr$ ag $\arr$ yahoo $\arr$ multirc $\arr$ boolq $\arr$ wic} \\
\midrule
11 & T5 & \makecell{
wsc $\arr$ banking77-19 $\arr$ banking77-9 $\arr$ banking77-8 $\arr$ banking77-25 $\arr$ \\ 
yahoo-1 $\arr$ 
banking77-34 $\arr$ banking77-3 $\arr$ banking77-23 $\arr$ \\
cb $\arr$ banking77-7 $\arr$ banking77-35 $\arr$ banking77-13 $\arr$ imdb $\arr$ \\ 
banking77-12 $\arr$ 
banking77-17 $\arr$ multirc $\arr$ banking77-14 $\arr$ emotion-0 $\arr$ \\
banking77-22 $\arr$ yp $\arr$ dbpedia-14-5  $\arr$ banking77-30 $\arr$ \\ 
banking77-1 $\arr$ 
banking77-15 $\arr$ boolq $\arr$ banking77-20 $\arr$ banking77-21 $\arr$ \\
dbpedia-14-2 $\arr$ qnli $\arr$ banking77-31 $\arr$ banking77-29 $\arr$ emotion-2 $\arr$ yahoo-3 $\arr$ \\
dbpedia-14-1 $\arr$ banking77-32 $\arr$ banking77-0 $\arr$ rte $\arr$ 
ag-news $\arr$ dbpedia-14-4 $\arr$ \\
banking77-2 $\arr$ yahoo-4 $\arr$ banking77-11 $\arr$ banking77-37 $\arr$ banking77-27 $\arr$ \\ sst2 $\arr$ 
banking77-33 $\arr$ copa $\arr$ 
banking77-5 $\arr$ dbpedia-14-0 $\arr$ wic  $\arr$ \\ 
qqp $\arr$ 
banking77-26 $\arr$ yahoo-2 $\arr$ banking77-10 $\arr$
banking77-36 $\arr$ \\ banking77-4 $\arr$ 
emotion-1 $\arr$ 
dbpedia-14-3 $\arr$ amazon $\arr$ 
banking77-28 $\arr$ \\ banking77-16 $\arr$
banking77-24 $\arr$ 
mnli $\arr$ cola $\arr$ \\ 
wnli $\arr$ 
banking77-18  $\arr$ banking77-6 $\arr$ dbpedia-14-6 $\arr$ yahoo-0
}\\
\midrule
12 & T5 & \makecell{banking77-29 $\arr$ yp $\arr$ banking77-30 $\arr$ banking77-26 $\arr$ \\
banking77-20 $\arr$ yahoo-2 $\arr$ 
amazon $\arr$ dbpedia-14-2 $\arr$ banking77-24 $\arr$ 
yahoo-3 $\arr$ \\
banking77-22 $\arr$ banking77-16 $\arr$ 
yahoo-0 $\arr$ dbpedia-14-1 $\arr$ emotion-2 $\arr$  
dbpedia-14-4 $\arr$\\ dbpedia-14-6 $\arr$ 
wic $\arr$ banking77-23 $\arr$ banking77-14 $\arr$ 
banking77-18 $\arr$ yahoo-4  $\arr$\\
banking77-5 $\arr$ banking77-0 $\arr$ 
banking77-13 $\arr$ 
cb $\arr$ banking77-35 $\arr$ rte $\arr$  \\
banking77-4 $\arr$ dbpedia-14-3 $\arr$ banking77-1 $\arr$ banking77-9 $\arr$\\ 
banking77-15 $\arr$ banking77-3 $\arr$ 
banking77-6 $\arr$ banking77-21 $\arr$ 
mnli $\arr$ banking77-2 $\arr$  \\
yahoo-1 $\arr$ boolq $\arr$
banking77-10 $\arr$ banking77-25 $\arr$ 
banking77-37 $\arr$ banking77-17 $\arr$  \\
qqp $\arr$ banking77-28 $\arr$
wnli $\arr$ banking77-8 $\arr$
banking77-31 $\arr$  \\
dbpedia-14-0 $\arr$ banking77-11  $\arr$ banking77-27 $\arr$
banking77-7 $\arr$ multirc $\arr$\\ 
banking77-33 $\arr$
banking77-12 $\arr$ imdb $\arr$ copa $\arr$ \\
banking77-19 $\arr$ cola $\arr$ 
banking77-34 $\arr$ sst2 $\arr$ emotion-0 $\arr$\\
wsc $\arr$ qnli $\arr$ emotion-1 $\arr$ 
banking77-32  $\arr$ dbpedia-14-5 $\arr$ ag-news $\arr$ banking77-36
}\\
\midrule
13 & T5 & \makecell{yahoo-2 $\arr$ copa $\arr$ banking77-22 $\arr$ emotion-0 $\arr$ banking77-1 $\arr$ emotion-1 $\arr$ \\
yahoo-0 $\arr$ banking77-32 $\arr$ banking77-37 $\arr$ dbpedia-14-0 $\arr$ banking77-3  $\arr$ qnli $\arr$ \\
multirc $\arr$ banking77-0 $\arr$ dbpedia-14-3 $\arr$ ag-news $\arr$ banking77-10 $\arr$ imdb $\arr$ \\ 
banking77-5 $\arr$ banking77-15 $\arr$ banking77-16 $\arr$ wnli $\arr$ \\
banking77-36 $\arr$ wsc $\arr$ banking77-13  $\arr$ banking77-19 $\arr$ amazon $\arr$ \\
banking77-29 $\arr$ banking77-33 $\arr$ boolq $\arr$ banking77-28 $\arr$ \\
yahoo-1 $\arr$ yp $\arr$ banking77-14 $\arr$ emotion-2 $\arr$ mnli $\arr$ banking77-7 $\arr$ \\
banking77-21 $\arr$ banking77-30 $\arr$ banking77-4 $\arr$ banking77-9 $\arr$ \\
banking77-35 $\arr$ dbpedia-14-5 $\arr$ 
banking77-26 $\arr$ \\
cola $\arr$ qqp $\arr$ yahoo-3 $\arr$ dbpedia-14-6 $\arr$ wic $\arr$ \\
banking77-25 $\arr$ banking77-31 $\arr$ 
banking77-17 $\arr$ \\
banking77-23 $\arr$ banking77-8 $\arr$ cb $\arr$ 
banking77-6 $\arr$ dbpedia-14-2 $\arr$ \\
banking77-20 $\arr$
dbpedia-14-1 $\arr$ yahoo-4 $\arr$ banking77-18 $\arr$ \\
banking77-2 $\arr$ banking77-34 $\arr$ 
banking77-12 $\arr$ dbpedia-14-4 $\arr$ banking77-27 $\arr$ \\
rte $\arr$ sst2 $\arr$ banking77-24 $\arr$ 
banking77-11 
}\\
\bottomrule
\end{tabular}
\caption{Thirteen different orders of task sequences used for continual learning experiments. Orders 1-7 correspond to the standard CL benchmarks adopted by prior works \citep{razdaibiedina2023progressive} for short-sequence experiments. Orders 8-10 are long-sequence orders spanning 15 tasks. Orders 11-13 are our customized extremely long sequences, where the tasks are \textbf{randomly permuted}. In these extremely long cases, existing techniques such as the SOTA, ProgPrompt \citep{razdaibiedina2023progressive}, cannot cope with these long tasks, due to the quadratic growing training and inference costs.}
\label{tab:seq}
\end{table*}


\begin{table*}[htbp]
\centering
\fontsize{9}{6}\selectfont
\scalebox{1}{
\begin{tabular}{ l|ccc|ccc }
 \toprule
 \textbf{Hyperparameter} $\downarrow$  & \multicolumn{3}{c|}{\textbf{Short-sequence benchmark}} & \multicolumn{3}{c}{\textbf{Long-sequence benchmark}} \\ 
 Num. samples $\rightarrow$ & 16& 200 & 1000 & 20 & 200 & 1000 \\
  \toprule
 \multicolumn{7}{c}{T5-large Model} \\
 \toprule
 Epochs & 300 & 150 & 20 & 300 & 150 & 20 \\ 
 Learning rate & 0.3 & 0.3 & 0.3 & 0.3 & 0.3 & 0.3 \\
 Length of shared prompt $\theta_{\Pc^{\ast}}$ & 10 & 10 & 10 & 10 & 10 & 10 \\
 Length of each prompt in $\Qc$& 10 & 10 & 10 & 10 & 10 & 10 \\
 Memory factor $\eta$ & 0.001 & 0.001 & 0.001 & 0.01 & 0.01 & 0.01 \\
 \toprule
 \multicolumn{7}{c}{BERT-base Model} \\
 \toprule
 Epochs & 300 & 150 & 40 & 300 & 150 & 40 \\ 
 Learning rate & 0.0001& 0.0001 & 0.0001 & 0.0001 & 0.0001 & 0.0001\\
 Length of shared prompt $\theta_{\Pc^{\ast}}$ & 10 & 10 & 10 & 5 & 5 & 5\\
 Length of each prompt in $\Qc$ & 10 & 10 & 10 & 5 & 5 & 5 \\
 Memory factor $\eta$ & 0.001 & 0.001 & 0.001 & 0.01 & 0.01 & 0.01 \\
 \bottomrule
\end{tabular}
}
\captionof{table}{Hyperparameters used for Q-tuning across different CL experiments.}
\label{tab:table_hpo}
\end{table*}

\begin{table*}[htbp]
\centering
\setlength{\tabcolsep}{4pt}
\fontsize{9}{6}\selectfont
\small
\scalebox{1}{
\begin{tabular}{c|ccc|ccc|ccc|ccc|ccc}
\toprule
 \multirow{3}{*}{\textbf{Sequence}}&\multicolumn{3}{c|}{\multirow{1}{*}{\textbf{Method}}}&\multicolumn{9}{c}{\textbf{T5-large Results}}\\
  & Q-prompt & Aggregation & $\theta_{\Pc^{\ast}}$&  \multicolumn{3}{c}{\textbf{Order1}}&\multicolumn{3}{c}{\textbf{Order2}}&\multicolumn{3}{c}{\textbf{Order3}}&\multicolumn{3}{c}{\textbf{Average}}\\
&\multicolumn{3}{c|}{(\textbf{Num. samples} $\rightarrow$)} & 16 & 200 & 1000 & 16 & 200 & 1000 & 16 & 200 & 1000& 16 & 200 & 1000 \\
\midrule
\multirow{3}{*}{Short}&\cmark & & & 74.1 & 80.0 & 79.6 & 74.2& 79.5 &79.9 & 75.3& 79.8& 80.1 & 74.5 &79.8&79.8\\ 
&\cmark& \cmark&  & 74.9 & 80.9& 80.4 & 75.1& 80.6& 80.1& 75.6& 81.1& 80.8 & 75.2 & 80.9 & 80.4\\
&\cmark& & \cmark & 75.0 & 80.7& 81.6 & 74.6& 80.7& 80.7& 75.7& 80.4& 80.6 & 75.1 & 80.6 &80.9\\
&\cmark& \cmark& \cmark&   75.8 & 81.2& 82.3 & 75.8& 81.1& 82.2 & 76.9& 81.1& 81.1 & \textbf{76.2}&\textbf{81.2}&\textbf{81.9}\\
\midrule
 \multirow{3}{*}{\textbf{Sequence}}&\multicolumn{3}{c|}{\multirow{1}{*}{\textbf{Method}}}&\multicolumn{9}{c}{\textbf{T5-large Results}}\\
  & Q-prompt & Aggregation & $\theta_{\Pc^{\ast}}$&  \multicolumn{3}{c}{\textbf{Order8}}&\multicolumn{3}{c}{\textbf{Order9}}&\multicolumn{3}{c}{\textbf{Order10}}&\multicolumn{3}{c}{\textbf{Average}}\\
&\multicolumn{3}{c|}{(\textbf{Num. samples} $\rightarrow$)} & 20 & 200 & 1000 & 20 & 200 & 1000 & 20 & 200 & 1000 &20 & 200 & 1000\\
\midrule
\multirow{3}{*}{Long}&\cmark & & & 76.3& 81.6& 81.0 & 76.9& 80.6& 80.5& 76.7& 80.1& 80.9 & 76.7 & 80.8 & 80.8  \\ 
&\cmark& \cmark&  &  77.1 & 81.6& 82.1 & 77.4& 81.7& 81.9 & 77.2& 80.2 & 82.4 & 77.2 & 81.1 & 82.1\\
&\cmark& & \cmark & 77.4 & 81.7& 82.5 & 77.9& 80.9& 82.5 & 77.1& 80.7& 82.0 &77.4  & 81.1 & 82.3\\
&\cmark& \cmark& \cmark&   78.3 & 82.4 & 83.5 & 79.7& 82.1 & 83.3& 78.7& 81.4& 83.1 & \textbf{78.9} & \textbf{81.9} & \textbf{83.3}\\
\bottomrule
\end{tabular}
}
\caption{ More details of the ablation study results on each order reported in Table~\ref{tab:abaltionstud_average}. For the long-sequence experiments, we set the queue size to 10. All results are averaged over 3 runs.}
\label{tab:more_details_ablation}
\end{table*}

\subsection{Implementation and Experiment Details}
\label{appendix:experiment}

\paragraph{More Details of the Methods for Comparison}
Following \cite{razdaibiedina2023progressive}, we consider 11 baseline methods for comparison with the proposed Q-tuning: 
\begin{itemize}[leftmargin=*]
    \item \textbf{Per-task Finetune} separately tunes the whole model for each task. 
    We use this type of method as a baseline in the short-sequence benchmark experiments. 
    \item \textbf{Continual Finetune}  \cite{wang2020efficient, huang2021continual}
    continually tunes the whole model on a sequence of tasks without adding any regularization or replaying data from the previous tasks. 
    \item \textbf{Prompt Tuning}  \citep{qin2021lfpt5,lester2021power} sequentially trains a shared soft prompt across all tasks, while freezing the pretrained model. 
    \item \textbf{Data Replay} finetunes the whole model for new tasks while replaying samples from previous tasks to prevent the CF problem. 
    \item \textbf{EWC} \citep{kirkpatrick2017overcoming} finetunes the  whole model using a regularization loss which penalizes updating parameters that could disturb the previously learned tasks. 
    \item \textbf{A-GEM}  \citep{chaudhry2018efficient} retrieves examples from old tasks and restricts the gradients to update the model when learning new tasks. 
    \item \textbf{LFPT5} \citep{qin2021lfpt5}  continuously trains a soft prompt that learns the tasks while generating samples for experience replay. 
    \item \textbf{MBPA++} \citep{de2019episodic} uses an episodic memory to augment BERT by storing all seen examples. 
    \item \textbf{IDBR} \citep{huang2021continual} 
    continuously trains the whole model by using data replay and a regularization loss. It adopts sentence representation disentanglement in task-specific and task-generic spaces, achieving SOTA on the CL benchmark with BERT. 
    \item \textbf{Per-task Prompt}  \citep{lester2021power} trains a separate soft prompt for each task while keeping the original model frozen. This type of method naturally eliminates the CF problem, because separately tuned prompts will not change when new tasks are learned. However, this independent prompt tuning setup cannot achieve forward knowledge transfer.
    \item \textbf{ProgPrompt} \citep{razdaibiedina2023progressive} trains a progressively increased prompt list to achieve the forward knowledge transfer and resist the CF problem using prompt tuning without relying on data replay. 
    Current SOTA on continual prompt tuning benchmarks with T5 and BERT. 
\end{itemize}

\paragraph{Implementation Details} We use PyTorch and HuggingFace Transformers library  for our implementation. For the standard CL benchmark, we use official datasets provided by \citet{zhang2015character}, following \citet{de2019episodic, zhang2015character}. We use HuggingFace datasets (\url{https://github.com/huggingface/datasets}) to download data for GLUE tasks \citep{wang2018glue}, SuperGLUE tasks \citep{wang2019superglue} tasks, IMDB movie reviews dataset \citep{maas2011learning}, Banking77 dataset \citep{casanueva2020efficient}, and Emotion dataset \citep{saravia2018carer}, which we use for long-sequence CL experiments, life-long learning experiments and ablation studies. Following previous studies \citep{ de2019episodic,razdaibiedina2023progressive}, for CL experiments, for each dataset, we use the available validation set as a test set (since test data is not available), and hold out 500 samples from the train set to construct the validation set. For our ablation studies, we report the maximal validation set performance.

We use the Adam optimizer and set the batch size to 8 for all the experiments. 
Following \cite{razdaibiedina2023progressive}, 
we train each prompt between 20 and 300 epochs, depending on the number of data points. We use the prompt checkpoints with the best validation set score as our final prompts. Prompts are initialized from randomly sampled tokens as in \citet{lester2021power}, hyperparameters are shown in the Table ~\ref{tab:table_hpo}.

The mutual information maximization can be approximated by maximizing its variational lower bound \citep{barber2004algorithm,poole2019variational} defined by Eq.~(\ref{mi_jsd}). But this variational approximation
requires extra costly computation to optimize the discriminator $\mathcal{F}_w$. We empirically find a KL-divergence based loss can go for the same goal, which is also verified by  \cite{muller2019does,tian2019contrastive}. 
The KL-divergence based MR loss between the new memory and the old memory is defined as follows:
\begin{align}
   \Lc_{\text{MR}} = \sum_{i\in|\Tc|}\sum_{(\x^i,\y^i)\in \Tc^i}D_{\rm KL}(p(\y^i|\x^i;\theta_{\Mcc}, \theta^i_{\Pc^{\ast}})&\, \nonumber\\
   \|\, p(\y^i|\x^i;\theta_{\Mcc}, \Wc^{i-1} \circ [\theta^{i-1}_{\Pc^{\ast}},\Qc^{i-1}]))&,
   \label{eq_loss_mr}
\end{align}
where only the shared prefix prompt $\theta^i_{\Pc^{\ast}}$ is trainable. This MR regularization loss does not require training an extra discriminator network, achieving the same objective as knowledge distillation \citep{hinton2015distilling}.

For all the CL experiments, we use early stopping as in \citet{huang2021continual}, to save model checkpoint based on the best validation performance on the current task. We report test set performance after training on all tasks as our final metric. For SuperGLUE experiments, we report maximal validation set performance over the course of training as in \citet{lester2021power}. We measure the validation performance after every epoch and use metrics described in Appendix~\ref{appendix:datasets}. We use the same hyperparameter settings for all prompt-based approaches (Q-tuning, Progressive Prompts, per-task prompt) as in \cite{razdaibiedina2023progressive}.


\paragraph{MLP-based prompt} We follow \citet{razdaibiedina2023progressive} by setting a two-layer MLP for parameterizing the soft-prompt. The two-layer MLP includes two linear layers with the ReLU activation function. The number of hidden nodes in the hidden layer is set to 512 in all Q-tuning experiments.

\section{More Ablation Study Results}
\label{appendix:more_results}

Table~\ref{tab:more_details_ablation} reports more details of the results on each order in Table~\ref{tab:abaltionstud_average} for the ablation study. Table~\ref{tab:eta_ablation} presents the effectiveness of setting different memory factors $\eta$ in the MR loss. As shown, the $\eta$ is suggested to $10^{-2}$ for the long sequence tasks. By comparing with the results of ``w/o MR'', the performance by using MR loss is improved by 1.7\% on average.

\begin{table}[H]
\centering
\setlength{\tabcolsep}{2pt}
\fontsize{9}{12}\selectfont
\small
\caption{Ablation study experiments (20 samples/class for long sequence) on the memory factor $\eta$ of the MR loss. All results are averaged over 3 runs.}
\scalebox{1}{
\begin{tabular}{c|cccc}
\toprule
 \multirow{2}{*}{\textbf{Parameter}}&  \multicolumn{4}{c}{\multirow{1}{*}{\textbf{Long Sequence}}}\\
 & \textbf{Order 8} & \textbf{Order 9}  & \textbf{Order 10} & \textbf{Average} \\
\midrule
$\eta=1$ & 73.5& 75.8& 73.2& 74.2\\
$\eta=10^{-1}$ & 77.1& 78.6& 77.3& 77.7\\
$\eta=10^{-2}$ & \textbf{78.3} & \textbf{79.7}& \textbf{78.7}& \textbf{78.9}\\
$\eta=10^{-3}$& 78.1& 79.4& 78.0& 78.5\\
$\eta=10^{-4}$& 77.8& 78.8& 77.8& 78.1\\
w/o MR & 77.3& 77.3& 77.1& 77.2\\
\bottomrule
\end{tabular}
}
\label{tab:eta_ablation}
\end{table}

\section{Evaluation of Forward Transfer and Backward Transfer}
We compare the forward transfer and backward transfer performance of Q-tuning with the competitors using the metrics defined by \citep{lopez2017gradient} in the long-sequence experiments. 
Table~\ref{tab:fkt1} and Table~\ref{tab:bkt1} report the forward knowledge transfer performance and backward knowledge transfer performance across 3 task orders (Order 8, Order 9, Order 10) of long-sequence experiments. 
Figures~\ref{fig:fwt_8_20}, \ref{fig:fwt_8_200} and \ref{fig:fwt_8_1000} show the forward transfer scores of the order 8 task sequence. Figures~\ref{fig:fwt_9_20}, \ref{fig:fwt_9_200} and \ref{fig:fwt_9_1000} show the forward transfer scores of the order 9 task sequence, and Figures~\ref{fig:fwt_10_20}, \ref{fig:fwt_10_200} and \ref{fig:fwt_10_1000} show the forward transfer scores of the order 10 task sequence.

\begin{table}[H]
\fontsize{9}{6}\selectfont
\begin{subtable}[c]{.48\textwidth}
  \centering
  \scalebox{0.87}{
  \begin{tabular}{ccc} 
    \toprule
      \textbf{Method}  &  Few-shot (20 samples/class) & Full-shot\\
      \midrule
      Finetune & 16.9	& 16.0 \\
      Data Replay  & 16.8 &	14.0 \\
      Prompt Tuning &  23.1& 24.9 \\
      Per-task Prompt &  0 &  0  \\
      LFPT5  & 18.9 &24.9 \\
      ProgPrompt$^{\ast}$ & 21.1& 24.7  \\
      \midrule 
      Q-tuning (Ours) &  \textbf{26.7}&	\textbf{29.7} \\
     \bottomrule
    \end{tabular}
    }
    \subcaption{The average forward knowledge transfer performance across 3 task orders (Order 8, Order 9, Order 10) of long-sequence experiments. }
    \label{tab:fkt1}
\end{subtable}
\begin{subtable}[c]{0.48\textwidth}
   \centering
  \scalebox{0.87}{
  \begin{tabular}{ccc} 
    \toprule
      \textbf{Method}  &  Few-shot (20 samples/class) & Full-shot\\
      \midrule
      Finetune & -59.5 & -63.5 \\
      Data Replay  & -24.7 & -18.0 \\
      Prompt Tuning & -47.9 & -71.0 \\
      Per-task Prompt &  0 &  0  \\
      LFPT5  & -13.5 & -8.8 \\
      ProgPrompt$^{\ast}$ & 0 & 0 \\
      \midrule 
      Q-tuning (Ours) &  \textbf{0} &	\textbf{0} \\
     \bottomrule
    \end{tabular}
    }
    \subcaption{The average backward knowledge transfer performance across 3 task orders (Order 8, Order 9, Order 10) of long-sequence experiments. }
    \label{tab:bkt1}
\end{subtable}
\caption{Average performance of the forward and backward knowledge transfer with the T5 model on the long-sequence benchmark.}
\end{table}

Figures~\ref{fig:bwt_8_20}, \ref{fig:bwt_8_200} and \ref{fig:bwt_8_1000} show the backward transfer scores of the order 8 task sequence, Figures~\ref{fig:bwt_9_20}, \ref{fig:bwt_9_200} and \ref{fig:bwt_9_1000} show the backward transfer scores of the order 9 task sequence, and Figures~\ref{fig:bwt_10_20}, \ref{fig:bwt_10_200} and \ref{fig:bwt_10_1000} show the backward transfer scores of the order 10 task sequence. In these backward transfer measurements, the score 0 stands for not forgetting old tasks. The evolution of the average
accuracy over learning new tasks \citep{lopez2017gradient} are reported in Figure~\ref{fig:evolution_avg}.

\begin{figure*}[htbp]
\centering
\setlength{\tabcolsep}{2pt}
\fontsize{9}{12}\selectfont
\small
\includegraphics[width=1\textwidth]{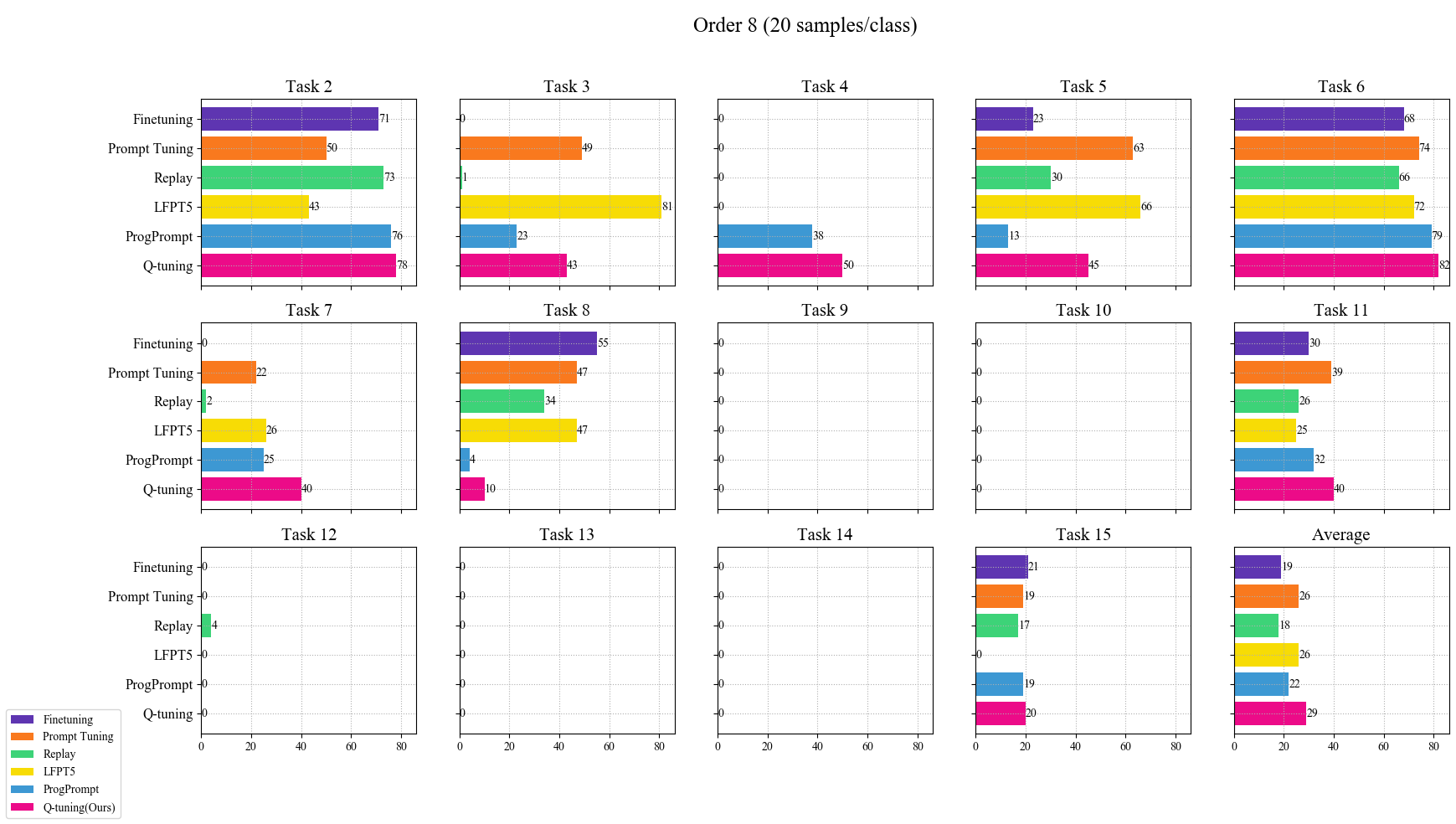}
\caption{Forward transfer score of different approaches on the order 8 (20 samples/class).}
\label{fig:fwt_8_20}
\end{figure*}

\begin{figure*}[htbp]
\centering
\setlength{\tabcolsep}{2pt}
\fontsize{9}{12}\selectfont
\small
\includegraphics[width=1\textwidth]{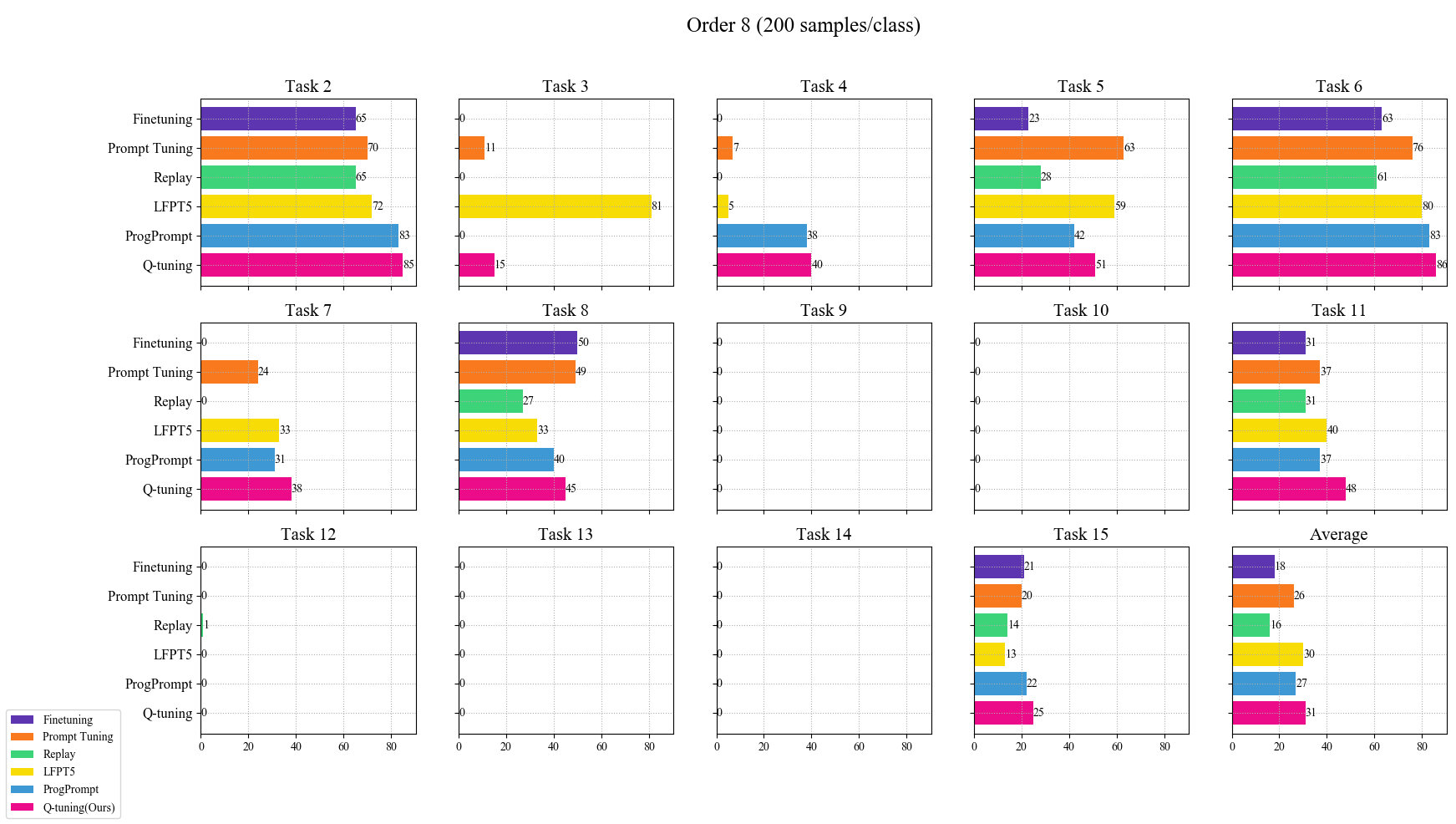}
\caption{Forward transfer score of different approaches on the order 8 (200 samples/class).}
\label{fig:fwt_8_200}
\end{figure*}

\begin{figure*}[htbp]
\centering
\setlength{\tabcolsep}{2pt}
\fontsize{9}{12}\selectfont
\small
\includegraphics[width=1\textwidth]{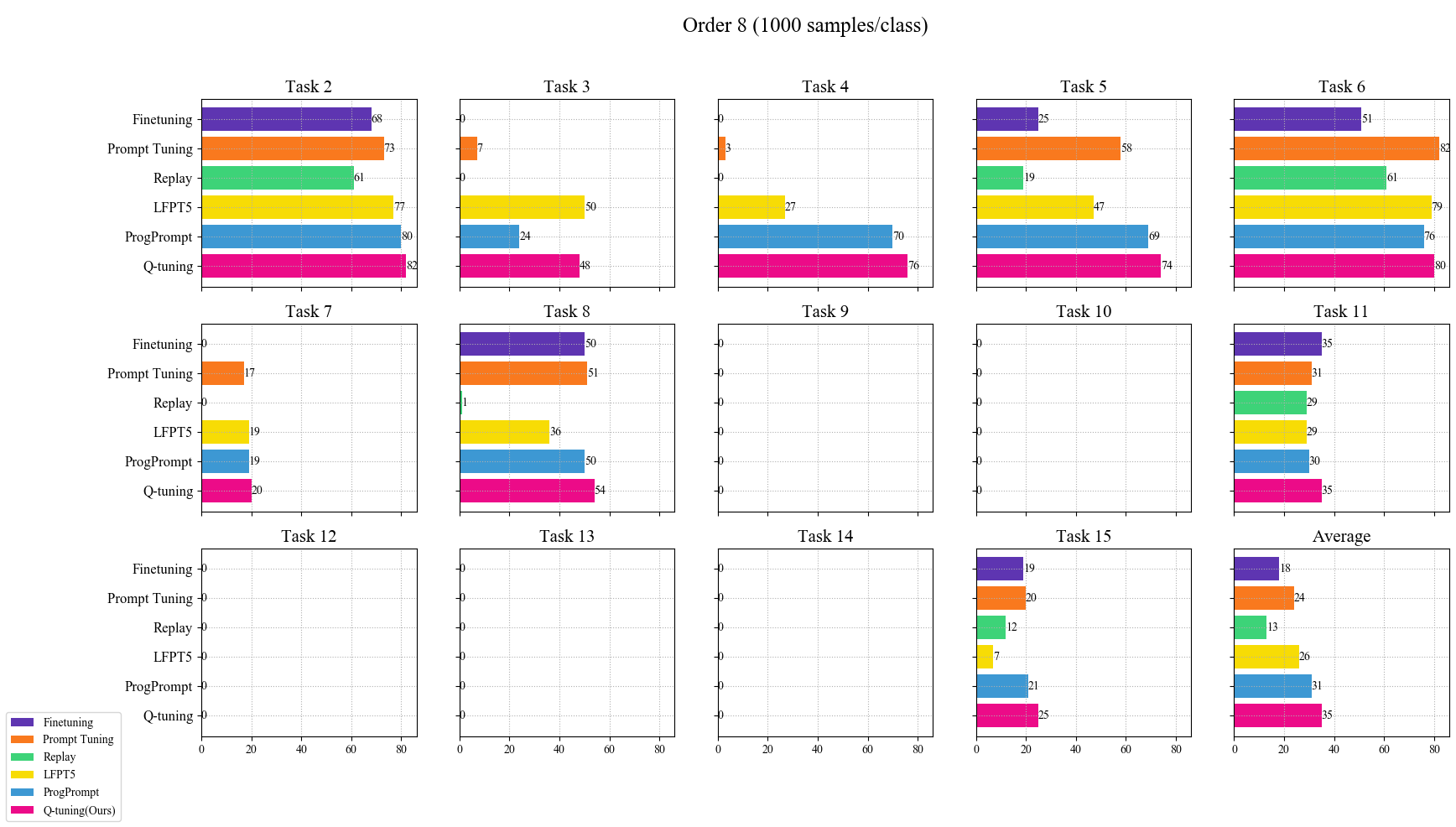}
\caption{Forward transfer score of different approaches on the order 8 (1000 samples/class).}
\label{fig:fwt_8_1000}
\end{figure*}

\begin{figure*}[htbp]
\centering
\setlength{\tabcolsep}{2pt}
\fontsize{9}{12}\selectfont
\small
\includegraphics[width=1\textwidth]{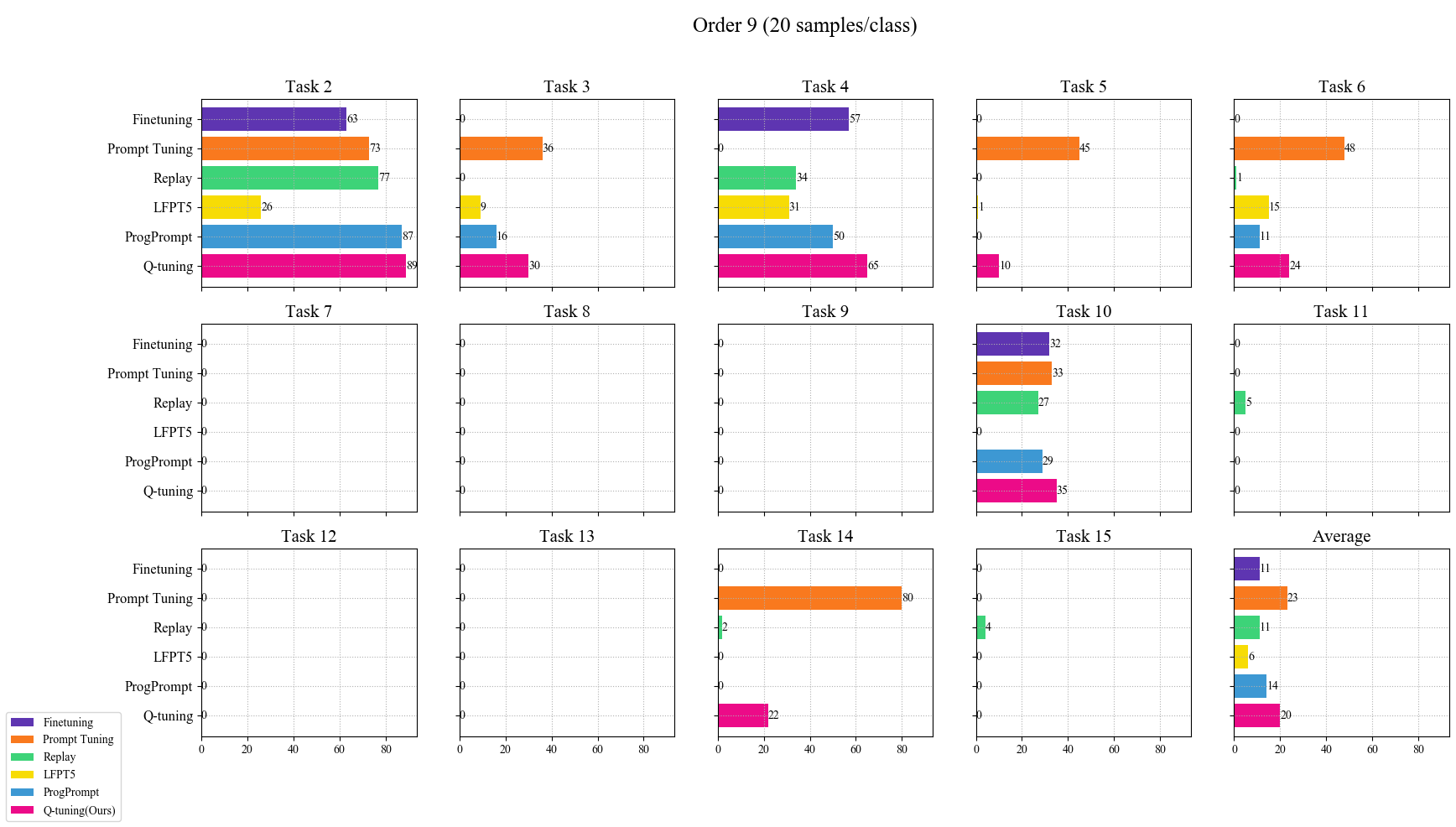}
\caption{Forward transfer score of different approaches on the order 9 (20 samples/class).}
\label{fig:fwt_9_20}
\end{figure*}

\begin{figure*}[htbp]
\centering
\setlength{\tabcolsep}{2pt}
\fontsize{9}{12}\selectfont
\small
\includegraphics[width=1\textwidth]{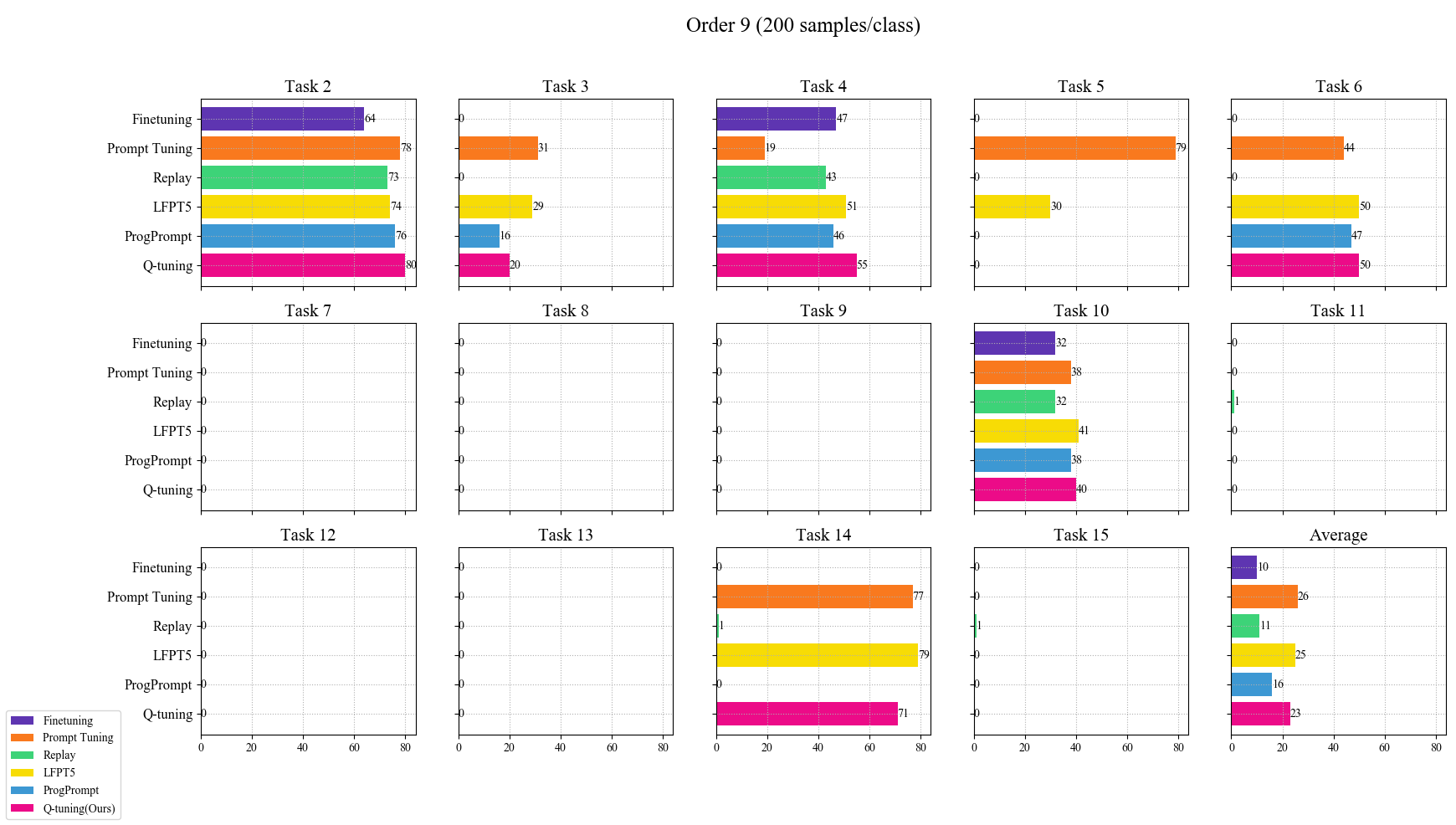}
\caption{Forward transfer score of different approaches on the order 9 (200 samples/class).}
\label{fig:fwt_9_200}
\end{figure*}

\begin{figure*}[htbp]
\centering
\setlength{\tabcolsep}{2pt}
\fontsize{9}{12}\selectfont
\small
\includegraphics[width=1\textwidth]{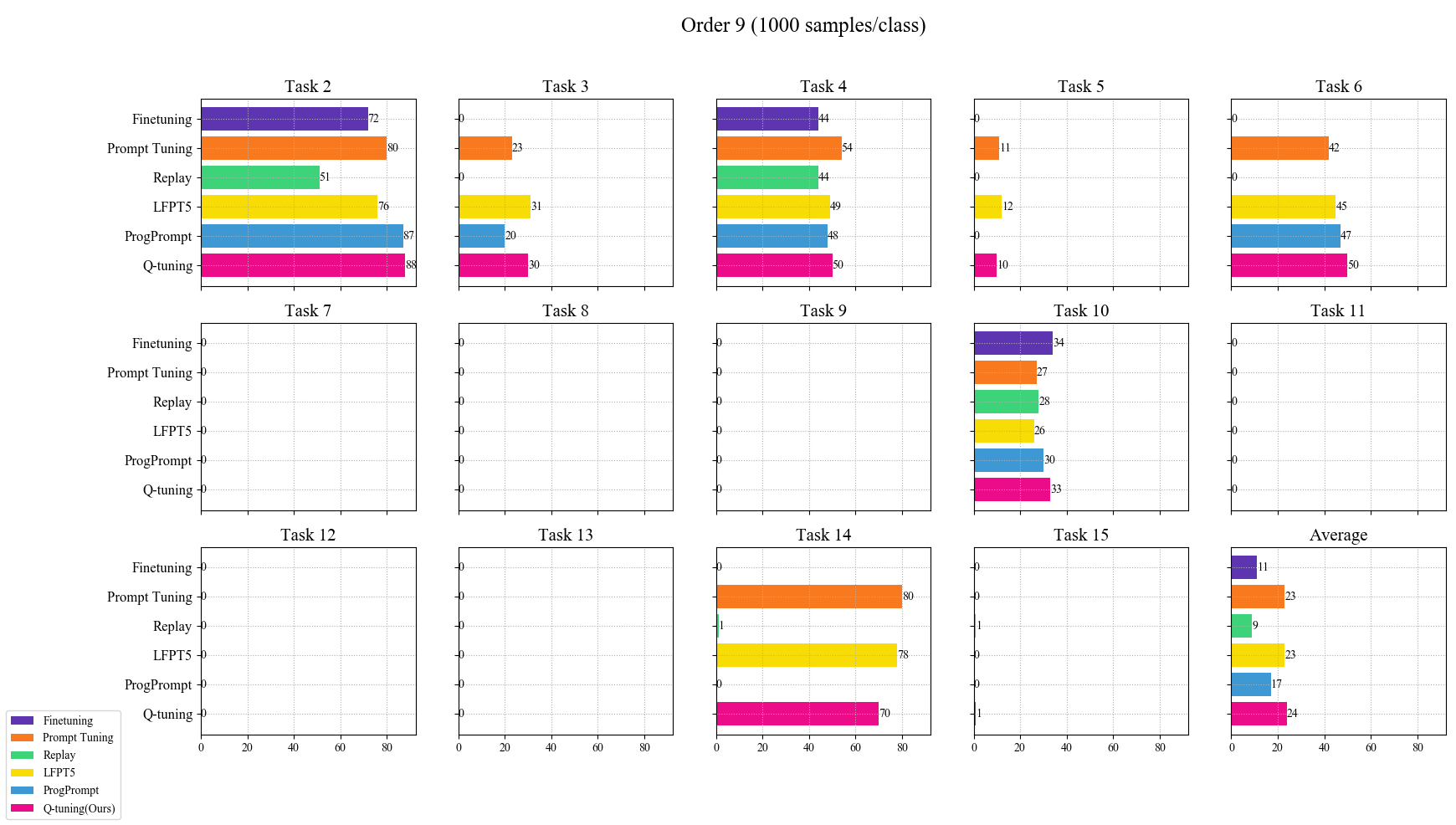}
\caption{Forward transfer score of different approaches on the order 9 (1000 samples/class).}
\label{fig:fwt_9_1000}
\end{figure*}

\begin{figure*}[htbp]
\centering
\setlength{\tabcolsep}{2pt}
\fontsize{9}{12}\selectfont
\small
\includegraphics[width=1\textwidth]{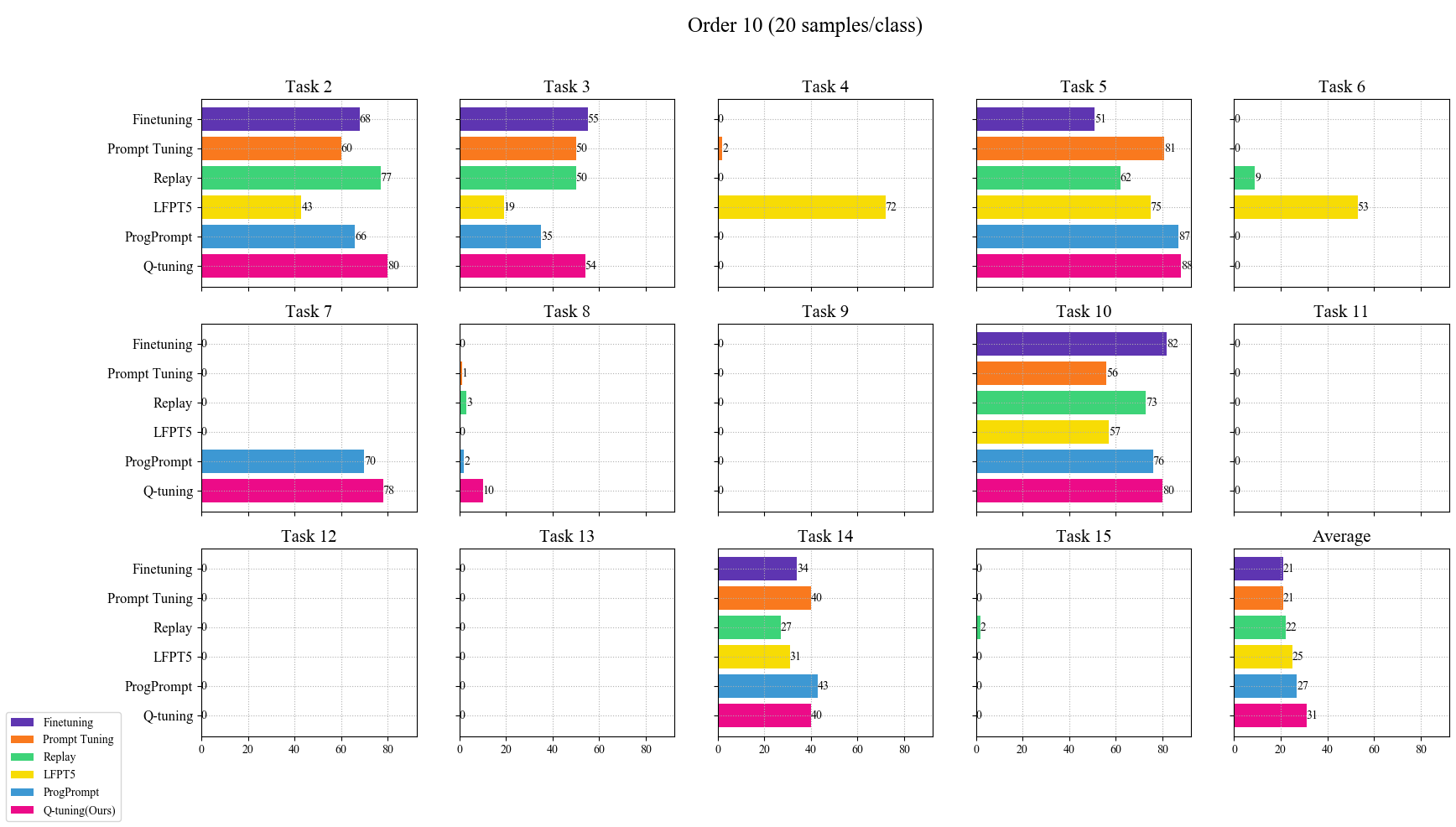}
\caption{Forward transfer score of different approaches on the order 10 (20 samples/class).}
\label{fig:fwt_10_20}
\end{figure*}

\begin{figure*}[htbp]
\centering
\setlength{\tabcolsep}{2pt}
\fontsize{9}{12}\selectfont
\small
\includegraphics[width=1\textwidth]{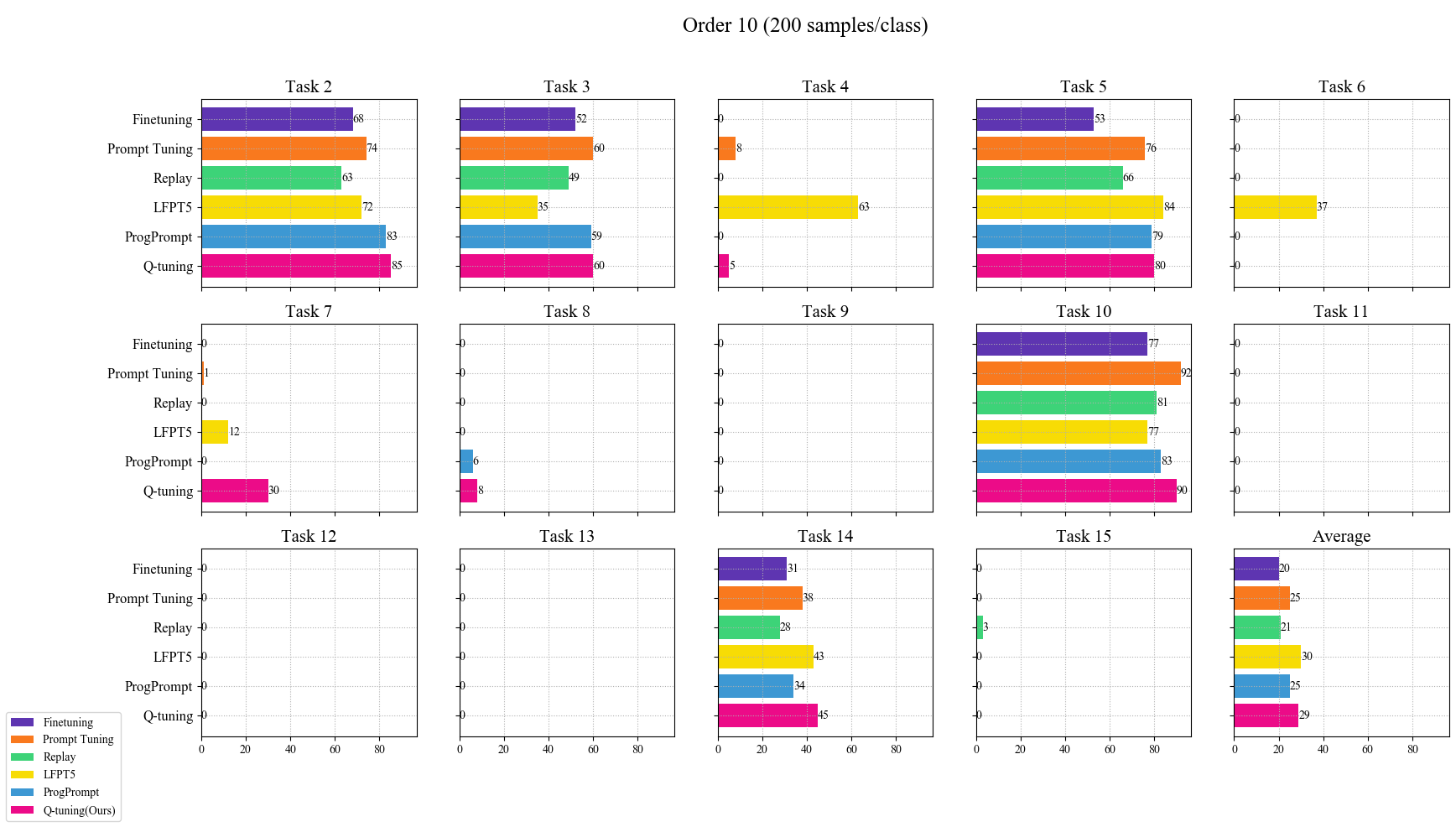}
\caption{Forward transfer score of different approaches on the order 10 (200 samples/class).}
\label{fig:fwt_10_200}
\end{figure*}

\begin{figure*}[htbp]
\centering
\setlength{\tabcolsep}{2pt}
\fontsize{9}{12}\selectfont
\small
\includegraphics[width=1\textwidth]{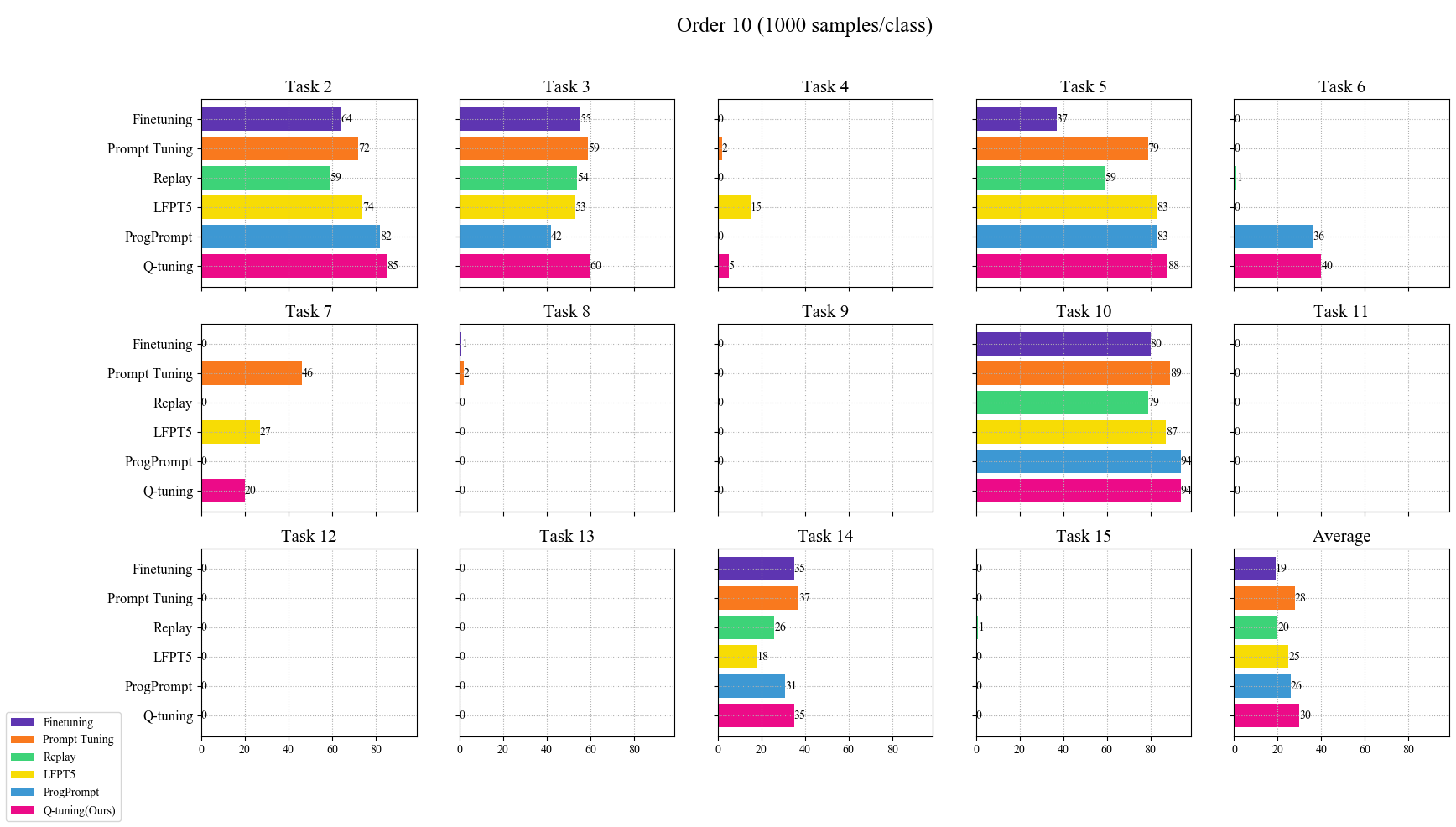}
\caption{Forward transfer score of different approaches on the order 10 (1000 samples/class).}
\label{fig:fwt_10_1000}
\end{figure*}

\begin{figure*}[htbp]
\centering
\setlength{\tabcolsep}{2pt}
\fontsize{9}{12}\selectfont
\small
\includegraphics[width=1\textwidth]{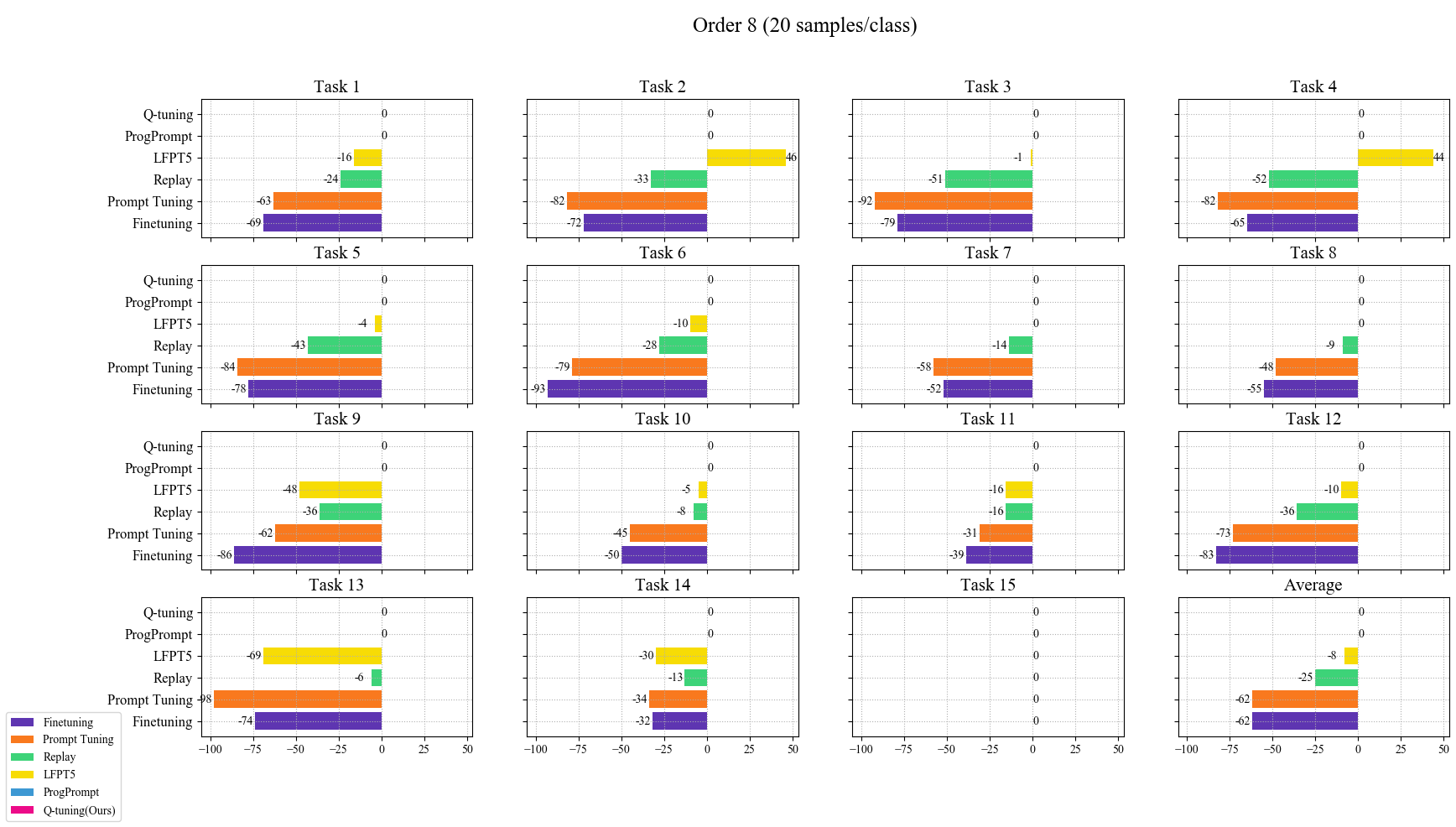}
\caption{Backward transfer score of different approaches on the order 8 (20 samples/class).}
\label{fig:bwt_8_20}
\end{figure*}

\begin{figure*}[htbp]
\centering
\setlength{\tabcolsep}{2pt}
\fontsize{9}{12}\selectfont
\small
\includegraphics[width=1\textwidth]{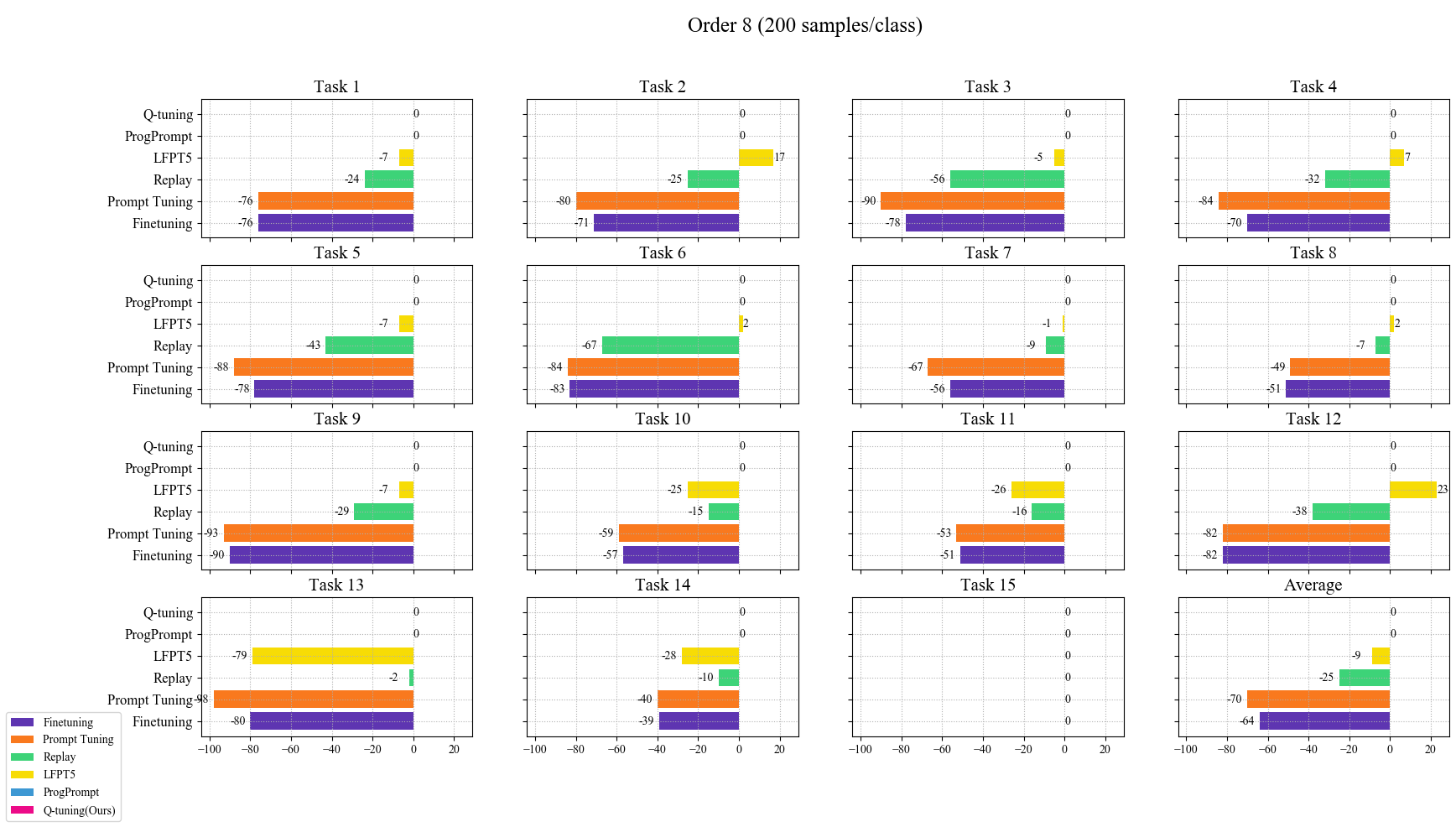}
\caption{Backward transfer score of different approaches on the order 8 (200 samples/class).}
\label{fig:bwt_8_200}
\end{figure*}

\begin{figure*}[htbp]
\centering
\setlength{\tabcolsep}{2pt}
\fontsize{9}{12}\selectfont
\small
\includegraphics[width=1\textwidth]{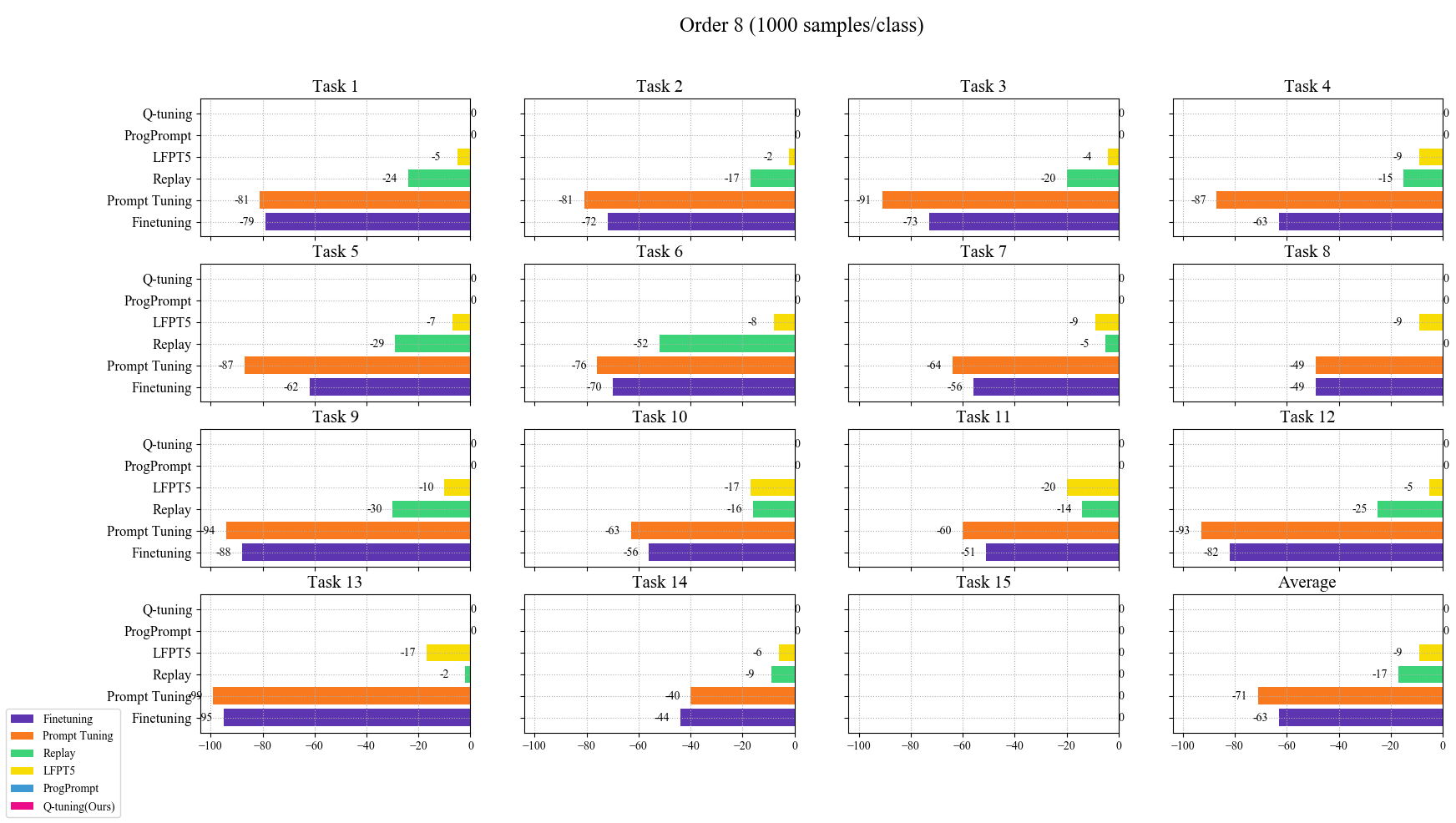}
\caption{Backward transfer score of different approaches on the order 8 (1000 samples/class).}
\label{fig:bwt_8_1000}
\end{figure*}

\begin{figure*}[htbp]
\centering
\setlength{\tabcolsep}{2pt}
\fontsize{9}{12}\selectfont
\small
\includegraphics[width=1\textwidth]{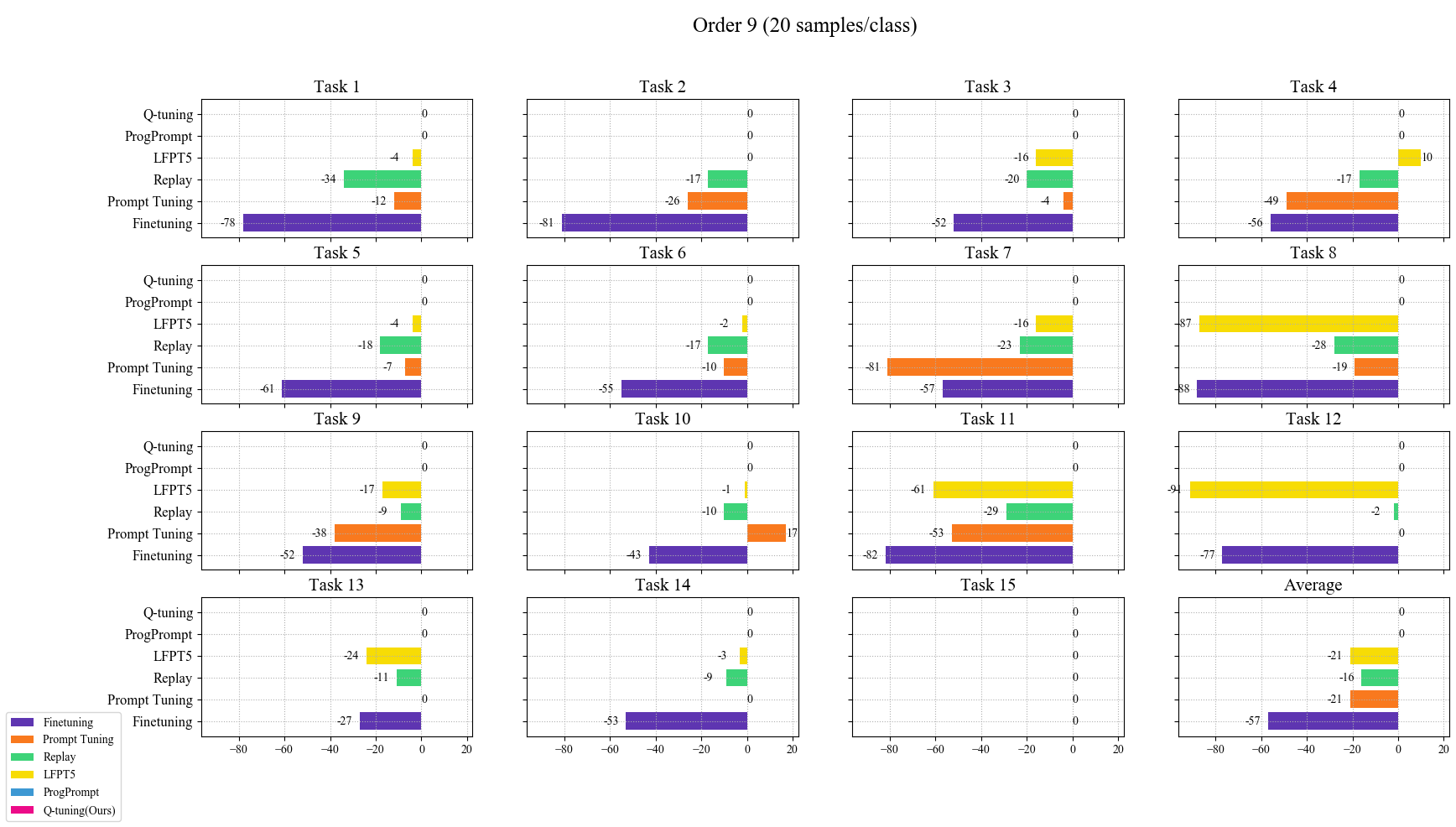}
\caption{Backward transfer score of different approaches on the order 9 (20 samples/class).}
\label{fig:bwt_9_20}
\end{figure*}

\begin{figure*}[htbp]
\centering
\setlength{\tabcolsep}{2pt}
\fontsize{9}{12}\selectfont
\small
\includegraphics[width=1\textwidth]{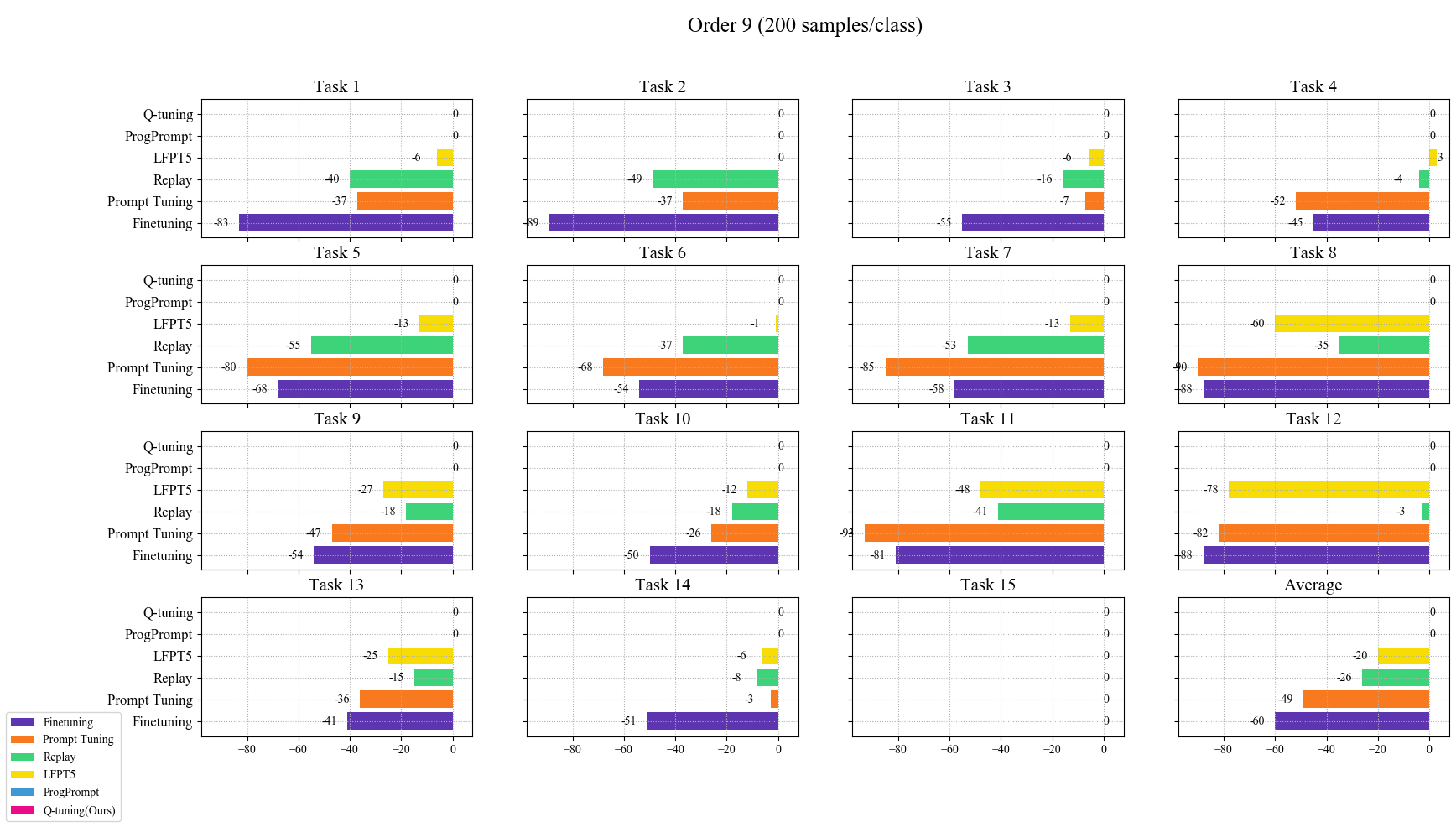}
\caption{Backward transfer score of different approaches on the order 9 (200 samples/class).}
\label{fig:bwt_9_200}
\end{figure*}

\begin{figure*}[htbp]
\centering
\setlength{\tabcolsep}{2pt}
\fontsize{9}{12}\selectfont
\small
\includegraphics[width=1\textwidth]{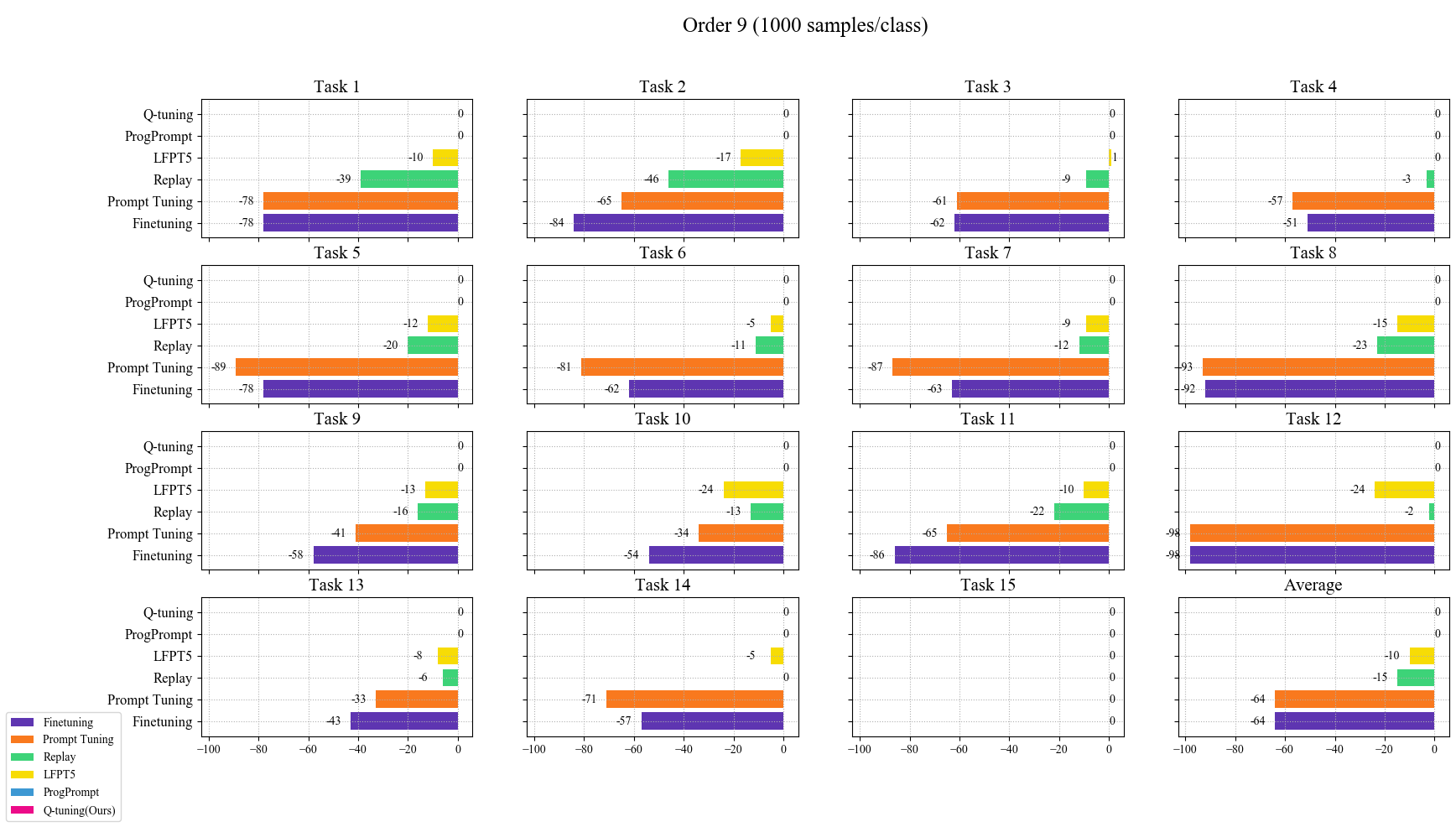}
\caption{Backward transfer score of different approaches on the order 9 (1000 samples/class).}
\label{fig:bwt_9_1000}
\end{figure*}

\begin{figure*}[htbp]
\centering
\setlength{\tabcolsep}{2pt}
\fontsize{9}{12}\selectfont
\small
\includegraphics[width=1\textwidth]{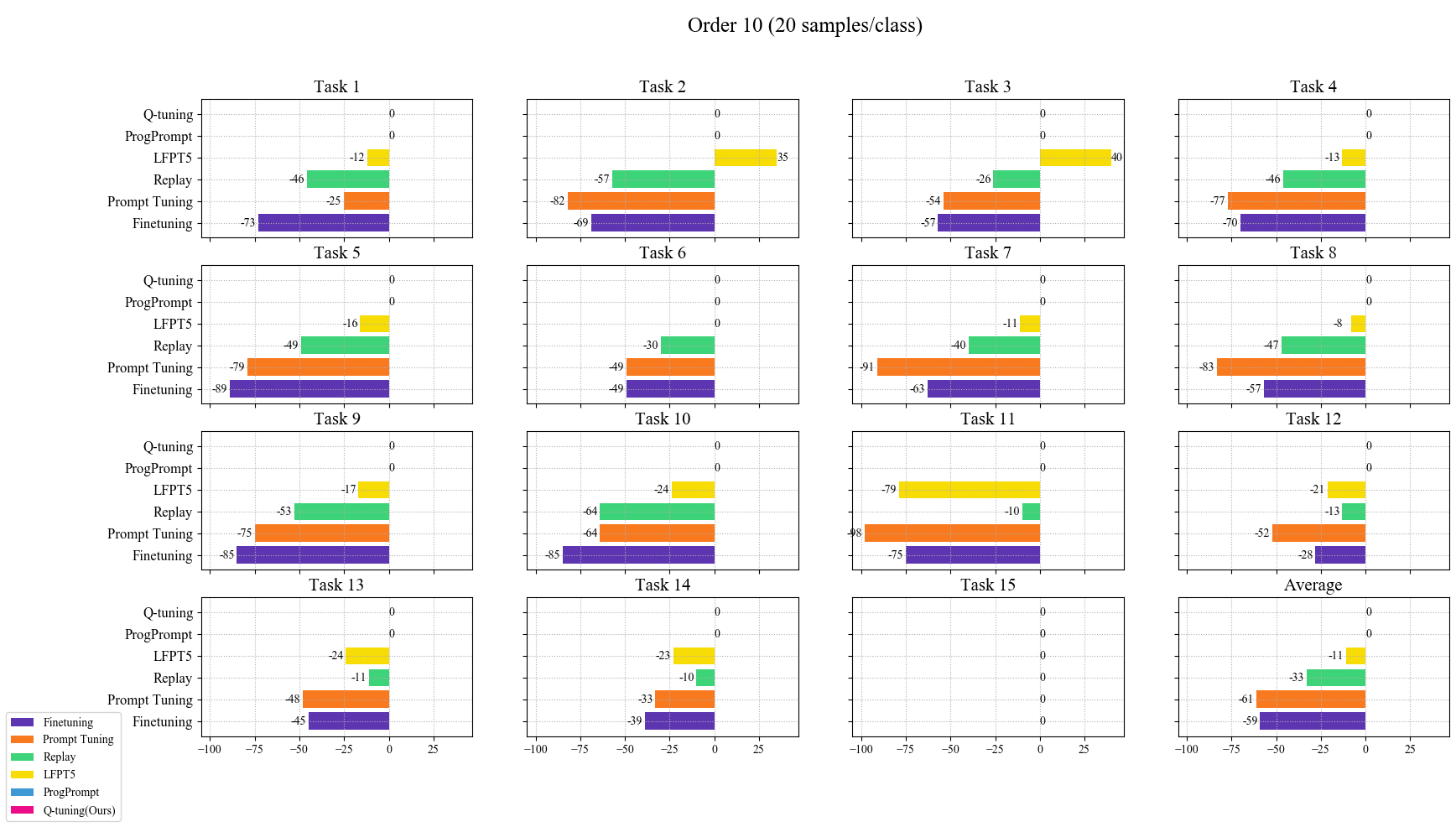}
\caption{Backward transfer score of different approaches on the order 10 (20 samples/class).}
\label{fig:bwt_10_20}
\end{figure*}

\begin{figure*}[htbp]
\centering
\setlength{\tabcolsep}{2pt}
\fontsize{9}{12}\selectfont
\small
\includegraphics[width=1\textwidth]{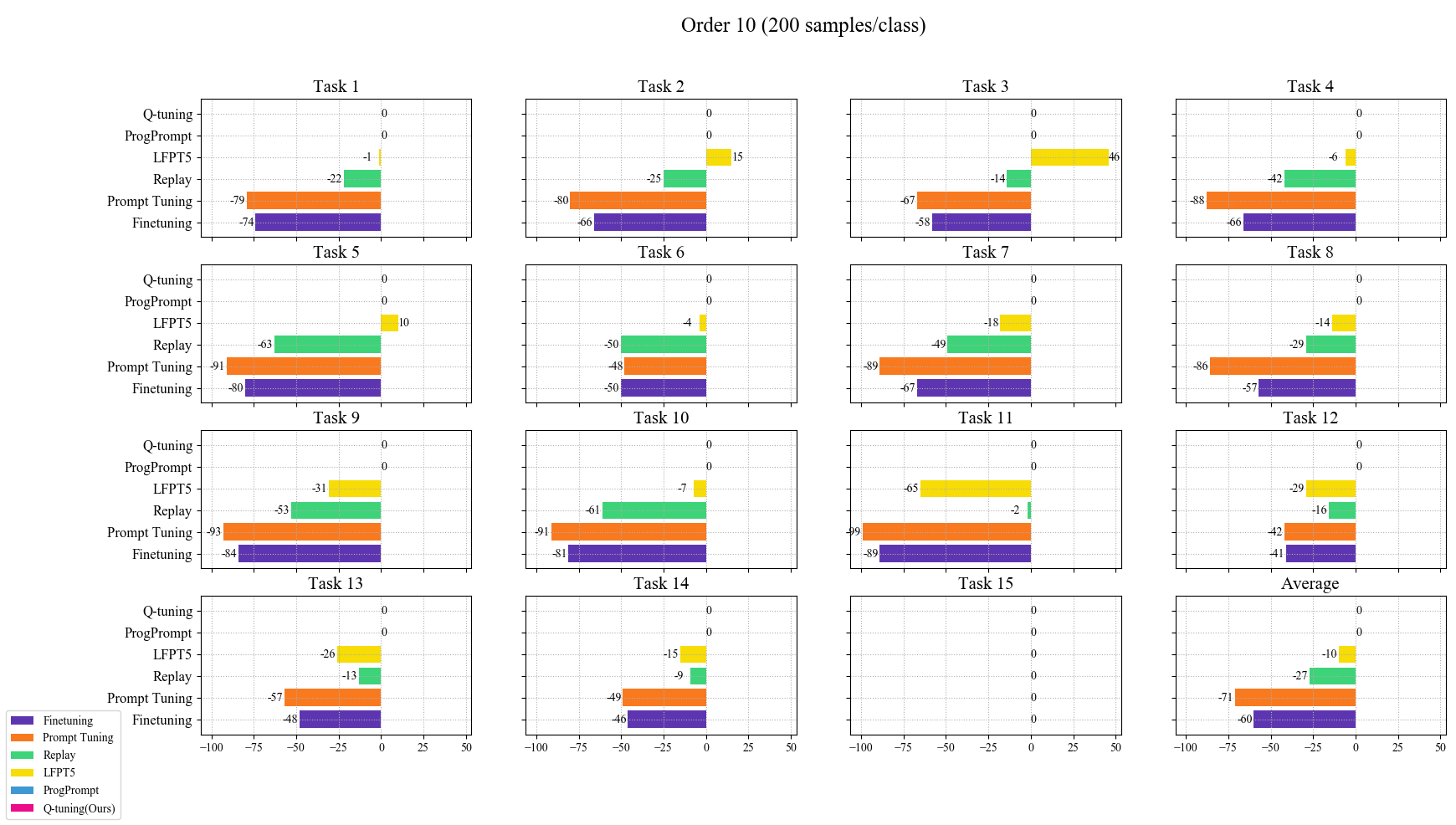}
\caption{Backward transfer score of different approaches on the order 10 (200 samples/class).}
\label{fig:bwt_10_200}
\end{figure*}

\begin{figure*}[htbp]
\centering
\setlength{\tabcolsep}{2pt}
\fontsize{9}{12}\selectfont
\small
\includegraphics[width=1\textwidth]{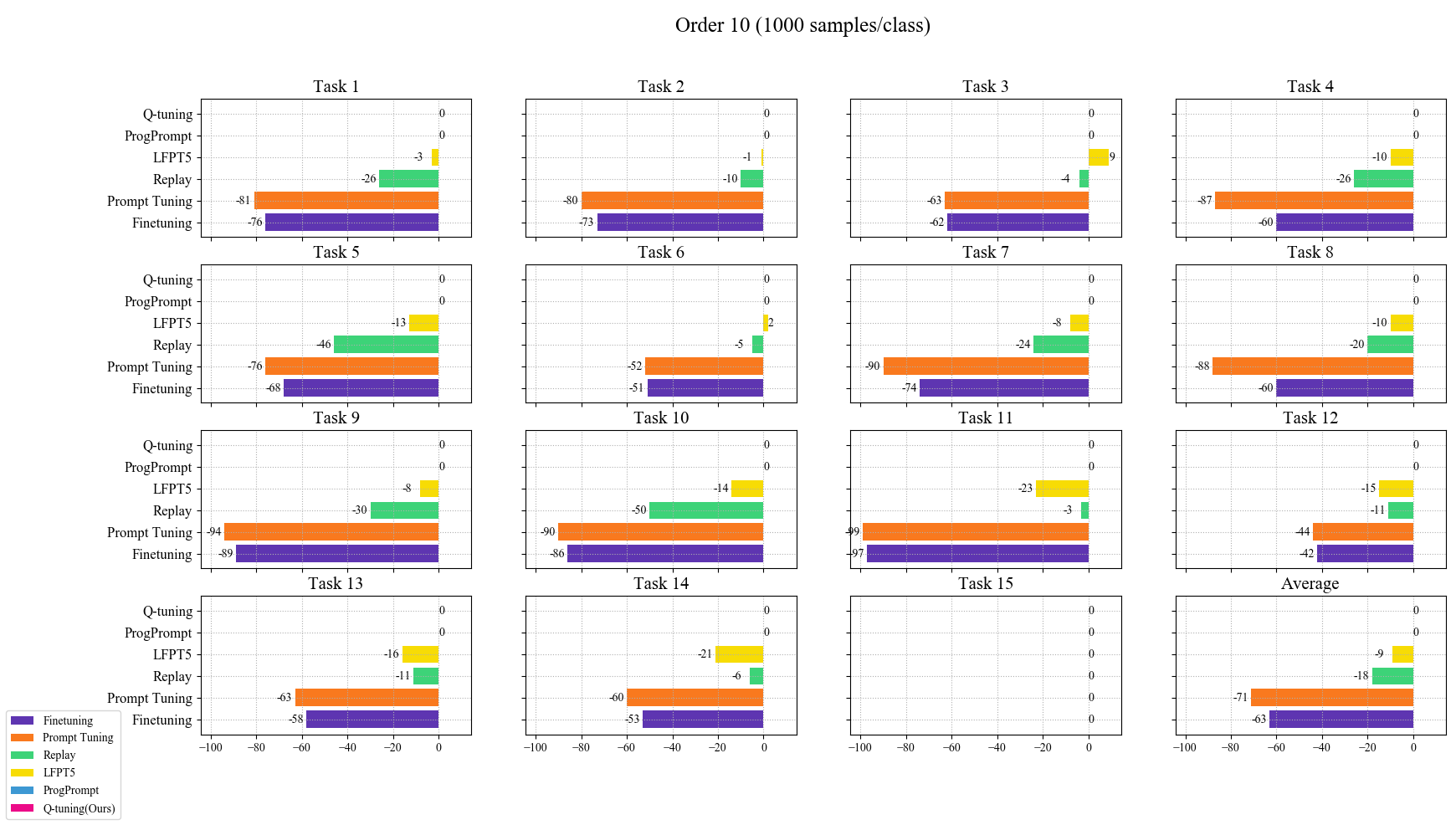}
\caption{Backward transfer score of different approaches on the order 10 (1000 samples/class).}
\label{fig:bwt_10_1000}
\end{figure*}

\begin{figure*}[htbp]
\centering
\setlength{\tabcolsep}{2pt}
\fontsize{9}{12}\selectfont
\small
\includegraphics[width=1\textwidth]{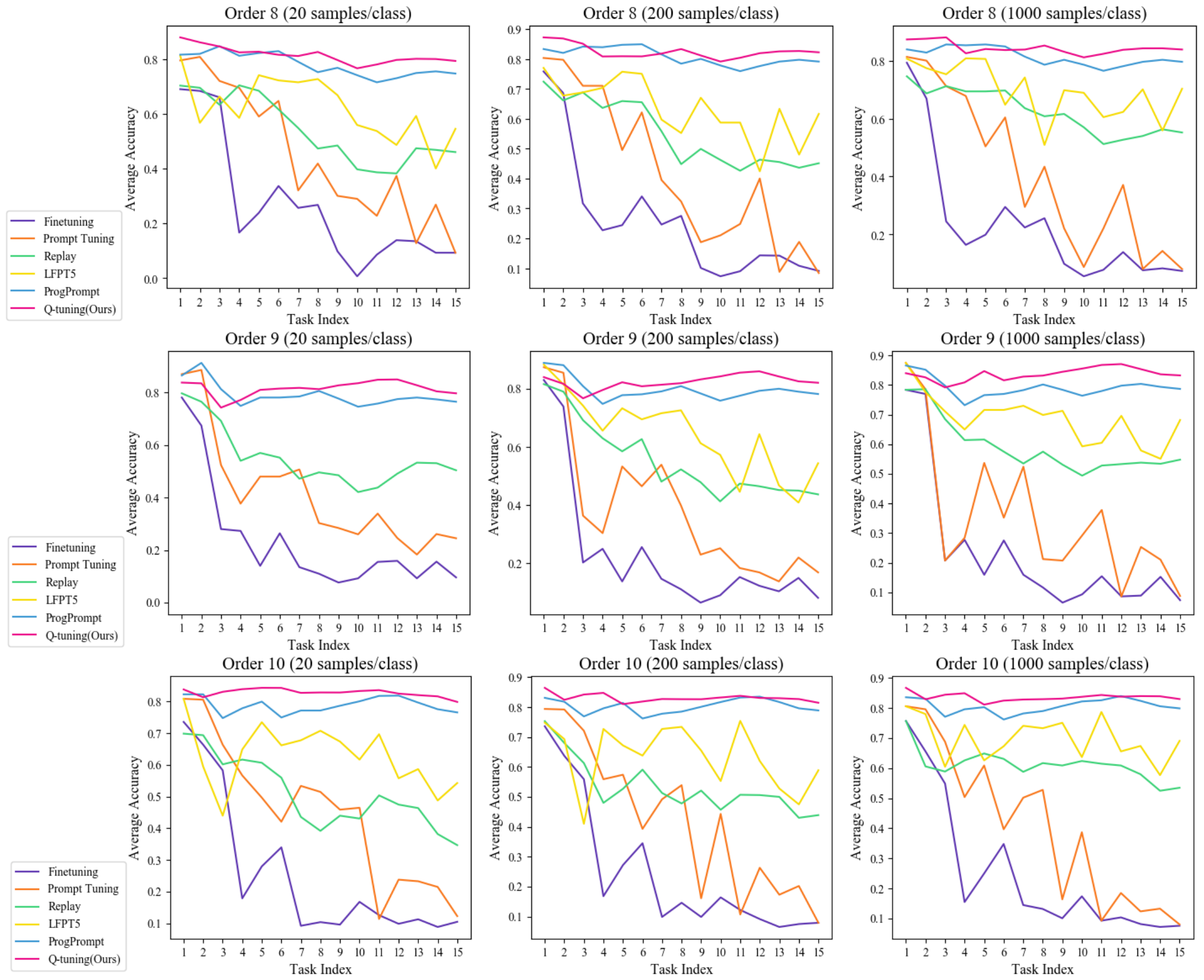}
\caption{Evolution of average accuracy after learning new tasks.}
\label{fig:evolution_avg}
\end{figure*}

\end{document}